\documentclass[11pt]{article}
\usepackage[letterpaper, margin=2cm]{geometry}

\usepackage[utf8]{inputenc} 
\usepackage[T1]{fontenc}    
\usepackage{titlesec}

\usepackage{settings/preamble}

\usepackage{upgreek}

\usepackage{bbm}
\newcommand{\ind}[1]{\mathbbm{1}\lbrace #1 \rbrace}

\newcommand{\piref}{\pi_{\mathrm{b}}}

\newcommand{\M}{\mathcal{M}}
\renewcommand{\S}{\mathcal{S}}

\newcommand{\A}{\mathcal{A}}

\renewcommand{\P}{\mathbb{P}}
\newcommand{\Ppi}{\mathbb{P}^{\pi}}
\renewcommand{\Q}{\mathbb{Q}}   
\newcommand{\Qpi}{\mathbb{Q}^{\pi}}  
\newcommand{\q}{q}

\renewcommand{\L}{\mathcal{L}}
\newcommand{\D}{\mathcal{D}}
\newcommand{\Err}{\mathcal{E}}
\renewcommand{\F}{\mathcal{F}}
\newcommand{\X}{\mathcal{X}}
\newcommand{\Y}{\mathcal{Y}}

\newcommand{\bx}{\bm{x}}
\newcommand{\by}{\bm{y}}
\newcommand{\bphi}{\bm{\phi}}
\newcommand{\btheta}{\bm{\theta}}
\newcommand{\bmu}{\bm{\mu}}
\newcommand{\bomega}{\bm{\omega}}

\newcommand{\bmupi}{\bmu^{\pi}}

\newcommand{\hbphi}{\hat{\bm{\phi}}}
\newcommand{\hbmu}{\hat{\bm{\mu}}}
\newcommand{\hbmupi}{\hat{\bmu}^{\pi}}
\newcommand{\hP}{\hat{\mathbb{P}}}
\newcommand{\hPpi}{\hat{\mathbb{P}}^{\pi}}
\newcommand{\hQpi}{\hat{\mathbb{Q}}^{\pi}}

\newcommand{\sbphi}{\bm{\phi}^{\star}}

\newcommand{\sbmupi}{\bm{\mu}^{\pi,\star}}
\newcommand{\fs}{f^{\star}}

\newcommand{\CP}{C_{\mathbb{P}}}
\newcommand{\Cpi}{C_{\infty}^{\pi}}
\newcommand{\Ccov}{C_{\mathrm{cov}}}
\newcommand{\Creg}{C_{\mathrm{reg}}}

\newcommand{\Div}{\mathsf{D}}

\newcommand{\replearn}{\textsc{RepLearn}\xspace}

\newcommand{\rhoP}{\rho_{\vphantom{\hP} \P}}

\newcommand{\tF}{\widetilde{\mathcal{F}}}
\newcommand{\tD}{\widetilde{\mathcal{D}}}

\newcommand{\hf}{\hat{f}}

\renewcommand{\Pr}{\mathrm{Pr}}
\newcommand{\Pneg}{P_{\mathrm{neg}}}

\newcommand{\Cd}{C_{d}}

\newcommand{\reg}{\mathrm{reg}}

\usepackage{environ}

\NewEnviron{equationfit}{%
\begin{equation*}
  \resizebox{\linewidth}{!}{\(
    \BODY
  \)}
\end{equation*}
}
\NewEnviron{alignfit}{%
\begin{equation*}
  \resizebox{\linewidth}{!}{\(
  \begin{aligned}
    \BODY
  \end{aligned}
  \)}
\end{equation*}
}
\usepackage{settings/definitions}

\usepackage[round]{natbib}

\usepackage[affil-it]{authblk}

\makeatletter
\def\section{\@startsection{section}{1}%
  {\z@}%
  {-3.0ex plus -0.5ex minus -0.5ex}%
  {1.5ex plus 0.5ex}%
  {\large\bfseries\raggedright}%
}
\def\subsection{\@startsection{subsection}{2}%
  {\z@}%
  {-2.4ex plus -0.4ex minus -0.4ex}%
  {1.2ex plus 0.4ex}%
  {\normalsize\bfseries\raggedright}%
}
\def\subsubsection{\@startsection{subsubsection}{3}%
  {\z@}%
  {-2.4ex plus -0.4ex minus -0.4ex}%
  {1.2ex plus 0.4ex}%
  {\normalsize\bfseries\raggedright}%
}
\def\paragraph{\@startsection{paragraph}{4}%
  {\z@}%
  {2pt \@plus 1pt \@minus 1pt}%
  {-1em}%
  {\normalsize\bfseries}%
}
\makeatother

\usepackage{xspace}
\newcommand{\algname}{\textsc{SpectralDICE}\xspace}

\begin{document}
  \allowdisplaybreaks[4]
  \setlength{\abovedisplayskip}{6pt minus 2pt}
  \setlength{\abovedisplayshortskip}{6pt minus 2pt}
  \setlength{\belowdisplayskip}{6pt minus 2pt}
  \setlength{\belowdisplayshortskip}{6pt minus 2pt}
  \setlength{\jot}{1pt}  
  \setlength{\floatsep}{1ex}
  \setlength{\textfloatsep}{1ex}

  \setlength{\medskipamount}{3.0pt plus 1.0pt minus 1.0pt}


  \title{Primal-Dual Spectral Representation for Off-policy Evaluation}
  \author[1]{Yang Hu$^*$} 
  \author[2]{Tianyi Chen\thanks{The first two authors contribute equally to this paper.}} 
  \author[1]{Na Li}
  \author[2]{Kai Wang}
  \author[2]{Bo Dai\thanks{Emails: Yang Hu (yanghu@g.harvard.edu), Tianyi Chen (tchen667@gatech.edu), Na Li (nali@seas.harvard.edu), Kai Wang (kwang692@gatech.edu), Bo Dai (bodai@cc.gatech.edu)}}
  \affil[1]{School of Engineering and Applied Sciences, Harvard University}
  \affil[2]{School of Computational Science and Engineering, Georgia Institute of Technology}
  \date{}

  \maketitle


  \begin{abstract}
  Off-policy evaluation (OPE) is one of the most fundamental problems in reinforcement learning (RL) to estimate the expected long-term payoff of a given target policy with \emph{only} experiences from another behavior policy that is potentially unknown. The distribution correction estimation (DICE) family of estimators have advanced the state of the art in OPE by breaking the \emph{curse of horizon}. 
  However, the major bottleneck of applying DICE estimators lies in the difficulty of solving the saddle-point optimization involved, especially with neural network implementations. In this paper, we tackle this challenge by establishing a \emph{linear representation} of value function and stationary distribution correction ratio, \ie, primal and dual variables in the DICE framework, using the spectral decomposition of the transition operator. Such primal-dual representation not only bypasses the non-convex non-concave optimization in vanilla DICE, therefore enabling an computational efficient algorithm, but also paves the way for more efficient utilization of historical data.
  We highlight that our algorithm, \algname, is the first to leverage the linear representation of primal-dual variables that is both computation and sample efficient, the performance of which is supported by a rigorous theoretical sample complexity guarantee and a thorough empirical evaluation on various benchmarks.
\end{abstract}
  
\section{Introduction}\label{sec:1-introduction}

The past decade has witnessed the ubiquitous success of reinforcement learning (RL) across various domains. 
Despite the original rationale that RL agents should learn a reward-maximizing policy from continuous interactions with the environment, there also exist a wide range of applicational scenarios where \emph{online} interaction with the environment may be expensive, inefficient, risky, unethical, and/or even infeasible, examples of which include robotics \citep{kalashnikov2018scalable, kahn2018composable}, autonomous driving \citep{shi2021offline, fang2022offline}, healthcare \citep{jagannatha2018towards, gottesman2018evaluating}, education \citep{mandel2014offline, slim2021markov}, dialogue systems \citep{jaques2019way, jiang2021towards} and recommendation systems \citep{li2011unbiased, chen2019top}. 
These application scenarios motivate the study of \emph{offline} RL, where the learning agent only has access to historical data collected by a separate behavior policy.

Off-policy evaluation (OPE) is one of the most fundamental problems in offline RL that aims at estimating the expected cumulative reward of a given target policy using only historical data collected by a different, potentially unknown behavior policy.
In the past decade, various off-policy performance estimators have been proposed \citep{hanna2019importance, xie2019towards, jiang2016doubly, foster2021offline}. However, these estimators generally suffer from the \emph{curse of horizon}~\citep{liu2018breaking}---step-wise variances accumulate in a multiplicative way, resulting in prohibitively high trajectory variances and thus unreliable estimators.
The recently proposed Distribution Correction Estimation (DICE) family of estimators have advanced the state of the art in OPE, leveraging the primal-dual formulation of policy evaluation for a saddle-point optimization approach that directly estimates the stationary distribution correction ratio, and hence breaking the curse of horizon \citep{nachum2019dualdice, nachum2019algaedice}.

Nevertheless, as systems scale up in terms of the size of state-action spaces, the saddle-point optimization in the formulation of DICE estimators become increasingly challenging to solve. Such \emph{curse of dimensionality} is common for RL methods in general, and people have been working to alleviate the computational burden by exploiting function approximators.
However, many known function approximators require additional assumptions to ensure computational and statistical properties~\citep{gordon1995stable,jiang2017contextual,chen2019information, zhan2022offline, che2024target}, which may not be easily statisfiable in practice. Moreover, the induced optimization upon function approximators may be difficult to solve~\citep{boyan1994generalization,baird1995residual,tsitsiklis1996analysis}.
In particular, under a generic neural network parametrization, computing the DICE estimator~\citep{nachum2020reinforcement} requires solving non-convex non-concave saddle-point optimizations, which is known to be NP-hard in theory and also yields unstable performance in practice, and is therefore regarded as intractable.  

This dilemma brings up a very natural question:
\begin{adjustwidth}{4em}{4em}
\begin{center}
  \itshape Can we design an OPE algorithm that is both \textbf{efficient} and \textbf{practical}?
\end{center}
\end{adjustwidth} 
By ``efficient'' we mean its statistical complexity avoids an exponential dependence on both the length of history and the dimension of state-action spaces, \ie, eliminating both \emph{curse of horizon} and \emph{curse of dimensionality}. 
By ``practical'' we mean the algorithm is free from unstable saddle-point optimizations and can be easily implemented and applied in practical settings. 


In this paper, we provide an \emph{affirmative} answer to this question by revealing a novel linear structure encapsulating both $Q$-functions and distribution correction ratios via a spectral representation of the transition operator, which has many nice properties to enable efficient representation learning and off-policy evaluation.

\paragraph{Contributions.} Specifically, the contributions of this paper can be summarized as follows:
\begin{itemize}
  \item We propose a novel \emph{primal-dual spectral representation} of the state-action transition operator, which makes both the $Q$-function and the stationary distribution correction ratio (\ie, the primal and dual variables in DICE) linearly representable in the primal/dual feature spaces, and thus enhances the tractability of the corresponding DICE estimator.

  \item We design \algname, an off-policy evaluation algorithm based on our primal-dual spectral representation, which bypasses the non-convex non-concave saddle-point optimization in vanilla DICE with generic neural network parameterization, and also makes efficient use of historical data. As far as we are concerned, our algorithm is the first to leverage the linear representation of both primal and dual variables that is \emph{computation and sample efficient}.

  \item The performance of the \algname~algorithm is justified both theoretically with a rigorous sample complexity guarantee and empirically by a thorough evaluation on various RL benchmarks. 
\end{itemize}

\vspace{-2pt}
\subsection{Related Work}
\vspace{-2pt}

\paragraph{Off-Policy Evaluation (OPE).} Off-policy evaluation has long been an active field of RL research. In the case where the behavior policy is known, various off-policy performance estimators have been proposed, including direct method (DM) estimators \citep{antos2008learning, le2019batch}, importance sampling (IS) estimators \citep{hanna2019importance, xie2019towards}, doubly-robust (DR) estimators \citep{dudik2011doubly, jiang2016doubly, foster2021offline} and other mixed-type estimators \citep{thomas2016data, kallus2020double, katdare2023marginalized}, which generally suffer from the \emph{curse of dimension}. 
In an effort to settle this issue, there is also abundant literature on estimating the correction ratio of the stationary distribution \citep{liu2018breaking, uehara2020minimax}, among which the distribution correction estimation (DICE) family of estimators are the state of the art that leverage a novel primal-dual formulation of OPE to eliminate the curse of horizon, and in the meantime, allow unknown behavior policies \citep{nachum2019dualdice, nachum2019algaedice, yang2022offline, zhang2020gendice,nachum2020reinforcement}.
However, as discussed above, the induced saddle-point optimization becomes unstable with neural networks, impeding the practical application of DICE estimators.

\paragraph{Spectral Representation in MDPs.}  
Spectral decomposition of the transition kernel is known to induce a linear structure of $Q$-functions, which enables the design of provably efficient algorithms assuming known (primal) spectral feature maps~\citep{jin2020provably, yang2020reinforcement, ren2022spectral}. 
These algorithms break the curse of dimensionality in the sense that their computation or sample complexity is independent of the size of the state-action space, but rather, only depends polynomially on the feature space dimension, the intrinsic dimension of the problem.


With the growing interest in spectral structures of MDPs, representation learning for RL has recently attracted much theory-oriented attention in the online setting~\citep{agarwal2020flambe, uehara2021representation}. 
Practical representation-based online RL algorithms have been designed via kernel techniques~\citep{ren2022free, ren2023stochastic}, latent variable models~\citep{ren2022latent,zhang2023provable}, contrastive learning~\citep{qiu2022contrastive, zhang2022making}, and diffusion score matching~\citep{shribak2024diffusion}. 
Recently, a unified representation learning framework is proposed from a novel viewpoint that leverages the spectral decomposition of the transition operator~\citep{ren2022spectral}.

Spectral representations have also been exploited in the offline setting~\citep{uehara2021pessimistic, ni2021learning, chang2022learning}, where the temporal difference algorithm is applied in the linear space induced by the primal spectral feature for estimating $Q$-functions.
The linear structure of the occupancy measure induced by the dual spectral feature is recently utilized in~\citet{huang2023reinforcement}, which leads to an offline RL algorithm for stationary density ratio estimation. 
Although the algorithm is theoretically sound, the stationary density ratio breaks the linearity in occupancy, and hence the algorithm is not computationally efficient. As far as we know, there is no such offline RL algorithm that efficiently utilizes both primal and dual representations. 
  \section{Preliminaries}\label{sec:2-settings}

\paragraph{Notations.} Denote by $\norm{\cdot}_p$ the $p$-norm of vectors or the $L^p$-norm of functionals, and by $\angl{\bm{x}, \bm{y}} = \bm{x}^{\top} \bm{y}$ the Euclidean inner product of vectors $\bm{x}$ and $\bm{y}$. Denote by $\hE[d^{\D}]{\cdot}$ the empirically approximated expectation using samples from dataset $\D \sim d^{\D}$. Denote by $\Delta(S)$ the set of distributions over set $S$, the element of which shall be regarded as densities whenever feasible. Denote the indicator function by $\ind{\cdot}$. Write $[n] := \set{1,\ldots,n}$ for $n \in \mathbb{Z}_{+}$. Regard $f(n) \lesssim g(n)$ as $f(n) = O(g(n))$.

\paragraph{Markov Decision Processes (MDPs).} We consider an \emph{infinite-horizon} discounted Markov decision process (MDP) $\M = (\S, \A, \P, r, \mu_0, \gamma)$, where $\S$ is the (possibly infinite) state space, $\A$ is the (possibly infinite) action space; $\P: \S \times \A \to \Delta(\S)$ is the transition kernel, $r: \S \times \A \to [0,1]$ is the reward function; $\mu_0 \in \Delta(\S)$ is the initial state distribution, and $\gamma \in (0,1)$ is the reward discount factor, so that the discounted cumulative reward can be defined as $\sum_{t=0}^{\infty} \gamma^t r_t$.
We consider \emph{stationary Markovian policies} $\varPi := \set{ \pi: \S \to \Delta(\A) }$ that admit an action distribution depending on the current state only. Given any policy $\pi \in \varPi$, let $\Es[\pi, \P]{\cdot}$ denote the expectation over the trajectory governed by $\pi$ and $\P$ (possibly under prescribed initial conditions).
Let $d_{\P}^{\pi}(\cdot, \cdot) \in \Delta(\S \times \A)$ denote the \emph{(stationary) state-action occupancy measure} under policy $\pi$, \ie, the normalized discounted probability of visiting $(s,a)$ in a trajectory induced by policy $\pi$, defined by
\begin{equation*}
  d_{\P}^{\pi}(s,a) = (1-\gamma) \E[\pi, \P]{ \sum_{t=0}^{\infty} \gamma^t \ind{s_t = s, a_t = a} }.
\end{equation*}
Similarly, let $d_{\P}^{\pi}(\cdot) \in \Delta(\S)$ denote the \emph{state occupancy measure} subject to the relation $d_{\P}^{\pi}(s,a) = d_{\P}^{\pi}(s) \pi(a|s)$. 
Further, define the state/state-action value functions (a.k.a. $V$- and $Q$-functions) as follows: 
\begin{align*}
  V_{\P}^{\pi}(s) &:= \E[\pi,\P]{\sum_{t=0}^{\infty} \gamma^t r(s_t, a_t) \;\middle|\; s_0 = s}, \\
  Q_{\P}^{\pi}(s,a) &:= \E[\pi,\P]{\sum_{t=0}^{\infty} \gamma^t r(s_t, a_t) \;\middle|\; s_0 = s,~ a_0 = a}.
\end{align*}
In this way, the value of policy $\pi$ in $\M$ is defined by
\begin{equation}\label{eq:policy_value}
  \rhoP(\pi) := (1-\gamma) \E[s \sim \mu_0]{V_{\P}^{\pi}(s)} \nonumber\\
  = (1-\gamma) \E[\begin{subarray}{l} s \sim \mu_0,\\ a \sim \pi(\cdot | s) \end{subarray}]{Q_{\P}^{\pi}(s,a)},
\end{equation}
where the factor $(1-\gamma)$ is introduced for normalization. We omit the subscript $\P$ when the context is clear.

\begin{remark}
  In order to better illustrate how the proposed method works in MDPs with continuous state-action spaces, we abuse the notation a bit to regard $\P$, $\pi$ and $d^{\pi}$ as \emph{densities}. Parallel results for the discrete case can be analogously derived without difficulties.
\end{remark}

\paragraph{The Primal-Dual Characterization of \titlemath{\rho(\pi)}.} Distribution Correction Estimation (DICE) \citep{nachum2020reinforcement} is a primal-dual-based method that evaluates the value of a given target policy $\pi$ in the offline setting, using the linear programming~(LP) formulation of policy values~\citep{puterman2014markov}. Specifically, it is known that we can equivalently characterize $\rho(\pi)$ defined in \eqref{eq:policy_value} by the \emph{primal LP}:
\begin{align}\label{eq:LP_formulation-primal}
  \min_{Q(\cdot, \cdot)}\quad& \E[\begin{subarray}{l} s \sim \mu_0,\\ a \sim \pi(\cdot | s) \end{subarray}]{Q(s,a)}, \nonumber\\
  \textrm{s.t.}\quad& Q(s,a) \geq r(s,a) + \gamma \E[\begin{subarray}{l} s' \sim \P(\cdot | s,a),\\ a' \sim \pi(\cdot | s') \end{subarray}]{Q(s',a')},~ \forall (s,a) \in \S \times \A.
\end{align}
Further, it can be shown that strong duality holds in \eqref{eq:LP_formulation-primal}, with Lagrangian multipliers exactly the state-action occupancy measures $d^{\pi}(\cdot,\cdot)$. We can therefore characterize $\rho(\pi)$ by the following \emph{primal-dual LP}:
\begin{equation}\label{eq:LP_formulation-primal_dual}
  \min_{Q(\cdot, \cdot)} \max_{d(\cdot,\cdot)}\quad (1-\gamma) \E[\begin{subarray}{l} s \sim \mu_0,\\ a \sim \pi(\cdot | s) \end{subarray}]{Q(s,a)} + \E[(s,a) \sim d^\pi(\cdot,\cdot)]{ r(s,a) + \gamma \E[\begin{subarray}{l} s' \sim \P(\cdot | s,a),\\ a' \sim \pi(\cdot | s') \end{subarray}]{Q(s',a')} }. 
\end{equation}
We highlight that this primal-dual LP formulation is favored in the offline RL setting in that historical experiences can be utilized to empirically approximate the expectations in \eqref{eq:LP_formulation-primal_dual} after some simple change-of-variables. In particular, for any measurable function $f(s,a)$, the importance sampling (IS) estimator for the expected value of $f(s,a)$ is given by 
\begin{equation}\label{eq:IS_change_of_variable}
    \E[(s,a) \sim d^{\pi}]{f(s,a)} = \E[(s,a) \sim d^{\piref}]{\frac{d^{\pi}(s,a)}{d^{\D}(s,a)} \cdot f(s,a)},
\end{equation}
where $\zeta(s,a) := \frac{d^{\pi}(s,a)}{d^{\D}(s,a)}$ is known as the \emph{stationary distribution correction ratio} for dataset $\D \sim d^{\D}$. 

The DICE family estimators~\citep{nachum2019dualdice, zhang2022efficient, dai2020coindice} is designed by plugging the IS expectation estimator \eqref{eq:IS_change_of_variable} into \eqref{eq:LP_formulation-primal_dual}, such that the stationary distribution correction ratio $\zeta(\cdot,\cdot)$ is parameterized along with the $Q$-function to formulate an optimization, with various regularization available~\citep{yang2020off}. It is evident that the DICE family estimators are applicable to the offline RL setting with unknown behavior policy. 

\paragraph{Spectral Representation.} We can always perform spectral decomposition of the dynamic operator to obtain a spectral representation of \emph{any} MDP \citep{ren2022spectral}. In particular, \emph{low-rank MDPs} refer to such MDPs with intrinsic finite-rank spectral representation structures that enable scalable RL algorithms, and are thus of theoretical interest \citep{yao2014pseudo, jin2020provably}. Formally, $\M$ is said to be a \emph{low-rank MDP} if there exists a \emph{primal} feature map $\bphi: \S \times \A \to \R^d$ and \emph{dual} features $\tilde{\bmu}: \S \to \R^d$, $\btheta_r \in \R^d$, such that $\P(s' | s,a) = \angl{\bphi(s,a), \tilde{\bmu}(s')}$, $r(s,a) = \angl{\bphi(s,a), \btheta_r}$, for any $s,s' \in \S$, $a \in \A$. Here both the primal feature $\bphi$ and the dual features $\tilde{\bmu}$, $\btheta_r$ are assumed to be
unknown, and thus must be learned from data~\citep{agarwal2020flambe,uehara2021representation}. 

Unfortunately, it is revealed in \citet{zhang2022making, ren2022spectral} that learning the features of a low-rank MDP is difficult from the unnormalized density fitting point of view.
To settle this tractability issue, the above papers propose a reparameterization of the dual feature as $\tilde{\bmu}(\cdot) = \q(\cdot) \bmu(\cdot)$, where we introduce an auxiliary distribution $\q(\cdot) \in \Delta(\S)$ that will be specified later. Therefore, we will stick to the following spectral decomposition of the transition kernel in this paper:
\begin{equation}\label{eq:spectral_representation_P}
  \P(s' | s,a) = \angl{\bphi(s,a), \q(s') \bmu(s')},~ \forall s \in \S, a \in \A, s' \in \S.
\end{equation}
Under such reparameterization, it has been shown that the spectral representaton can be learned efficiently.

Additionally, we also assume $\mu_0$ to be linearly representable in the dual feature space.

\begin{assumption}[initial representation]\label{assum:initial_distribution}
  There exists $\bomega_0 \in \R^d$, such that $\mu_0(s) = \q(s) \angl{\bmu(s), \bomega_0}$, $\forall s \in \S$.
\end{assumption}

\paragraph{Off-Policy Evaluation (OPE).} We consider a setting where we are given $\D = \set{(s_i,a_i,s'_i) \mid i \in [N]}$, an offline dataset of $N$ historical transitions, sampled by certain \emph{behavior policy} $\piref$ that could be unknown. The objective is to estimate the expected cumulative rewards $\rho(\pi)$ of a different \emph{target policy} $\pi$.

For satisfactory performance, it is important that the behavior policy provides sufficient data coverage for the frequent transitions experienced by policy $\pi$. Specifically, we assume the occupancy ratio between $\pi$ and $\piref$ satisfies the following regularity assumption.

\begin{assumption}[concentratability]\label{assum:sufficient_sampling}
  $\frac{d^{\pi}(s,a)}{d^{\piref}(s,a)} \leq \Cpi$, $\forall s \in \S, a \in \A$.
\end{assumption}

We point out that the concentratability assumption is standard in offline RL literature \citep{munos2008finite, chen2019information}, and is also implicitly enforced in recent work like \citet{huang2023reinforcement} (see Definition 1 therein). We are aware that the coefficient $\Cpi$ can potentially be translated into different feature-related constants \citep{uehara2021representation}, which does not change the asymptotics of sample complexity, yet only adds to the technical complexity. For clarity, we will stick to the simple \Cref{assum:sufficient_sampling} in this paper.
  \section{\algname: OPE using Primal-Dual Spectral Representation}\label{sec:3-algorithm}

In this section, we first introduce a novel linear representation for the stationary distribution correction ratio using the \emph{dual} spectral feature of transition kernel. We highlight that this linear structure, together with the known linear representation of $Q$-functions, helps to bypass the non-convex non-concave optimization required in the computation of DICE estimators, and also enables efficient utilization of historical data sampled by unknown behavior policies. Based on the above ideas, we present \algname, the proposed off-policy evaluation (OPE) algorithm using our primal-dual spectral representation.

\subsection{Primal-Dual Spectral Representation}\label{sec:3-algorithm-1-representation}

We start by specifying the primal-dual spectral representation used in \algname. At first glance, it may seem natural to directly learn the spectral representation of $\P$ as defined in \eqref{eq:spectral_representation_P}. However, it turns out that this naive approach includes the target policy $\pi$ in the linear representation of $d^{\pi}(\cdot, \cdot)$, which in turn induces a complicated representation for the stationary distribution correction ratio $\zeta(\cdot,\cdot)$~\citep{huang2023reinforcement}, and thus, leads to an intractable optimization \eqref{eq:LP_formulation-primal_dual} for the computation of the DICE estimator.  

The above challenge inspires us to properly reparameterize the spectral decomposition \eqref{eq:spectral_representation_P}. Specifically, since we only work with a fixed target policy $\pi$ for off-policy evaluation, we shall consider the following alternative representation of the state-action transition kernel $\Ppi(s',a' | s,a) := \P(s' | s,a) \pi(a' | s')$:
\begin{equation}\label{eq:spectral_representation_P_pi}
  \Ppi(s',a' | s,a) = \angl[\Big]{\bphi(s,a), \q(s') \piref(a' | s') \underbrace{ \tfrac{\pi(a' | s')}{\piref(a' | s')} \bmu(s') }_{ \bmupi(s',a') } }.
\end{equation}
Note that \Cref{assum:sufficient_sampling} guarantees a non-zero denominator when the nominator is non-zero. We refer to \eqref{eq:spectral_representation_P_pi} as the \emph{primal-dual spectral representation} of the state-action) transition kernel $\Ppi$, where $\bphi(\cdot, \cdot)$ and $\bmupi(\cdot, \cdot)$ are still called \emph{primal} and \emph{dual} spectral features, respectively. The superscript $\pi$ of the dual spectral feature emphasizes its dependence on the target policy.

The primal-dual spectral representation has several nice properties. In particular, we can show that the $Q$-function $Q^{\pi}(s,a)$, the state-action occupancy measure $d^{\pi}(s,a)$, and the stationary distribution correction ratio $\zeta(s,a)$ can all be represented in linear forms using the primal/dual features, as summarized below.

\begin{lemma}\label{thm:linear_representation_Q_d}
  With primal-dual spectral representation \eqref{eq:spectral_representation_P_pi}, the $Q$-function $Q^{\pi}(\cdot, \cdot)$ is linearly representable in the primal feature space with cofactor $\btheta_Q^{\pi} \in \R^d$:
  \begin{equation}\label{eq:linear_q}
    Q^{\pi}(s,a) = \angl{\bphi(s,a), \btheta_Q^{\pi}},~ \forall s \in S, a \in \A.
  \end{equation}
  Further, under \Cref{assum:initial_distribution}, the state-action occupancy measure $d^{\pi}(\cdot, \cdot)$ is also linearly representable in the dual feature space with cofactor $\bomega_d^{\pi} \in \R^d$:
  \begin{equation*}
    d^{\pi}(s,a) = \q(s) \piref(a | s) \angl{\bmupi(s,a), \bomega_d^{\pi}},~ \forall s \in S, a \in \A.
  \end{equation*}
  Specifically, when the auxiliary distribution $q(\cdot)$ is selected as the state-occupancy measure $d^{\piref}(\cdot)$ of the behavior policy $\piref$, the stationary distribution correction ratio can also be linearly represented as:
  \begin{equation}\label{eq:linear_zeta}
    \zeta(s,a) = \frac{d^{\pi}(s,a)}{\q(s) \piref(a | s)}
    = \angl{\bmupi(s, a), \bomega_d^{\pi}}. 
  \end{equation}
\end{lemma}

\begin{proof}[Proof]
  Note that the original dual feature in \eqref{eq:spectral_representation_P} can be restored by $\bmu(s') = \frac{\piref(a' | s')}{\pi(a' | s')} \bmupi(s',a')$ for any $a' \in \A$. Then by Bellman recursive equation we have:
  \begin{align*}
    Q^{\pi}(s,a) &= \angl{\bphi(s,a), \btheta_r} + \gamma \int V^{\pi}(s') \angl{\bphi(s,a), \q(s') \bmu(s')} \diff s' \\
    &= \angl*{\bphi(s,a), \underbrace{\btheta_r + \gamma \int V^{\pi}(s') \q(s') \bmu(s') \diff s'}_{\btheta_Q^{\pi}}}.
  \end{align*}
  Similarly, by the recursive property of $d^{\pi}$ we have:
  \begin{align*}
    d^{\pi}(s,a)
    ={}& (1-\gamma) \mu_0(s) \pi(a | s) + \gamma \int d^{\pi}(\tilde{s}, \tilde{a}) \Ppi(s,a | \tilde{s},\tilde{a}) \diff \tilde{s} \diff \tilde{a} \\
    ={}& (1-\gamma) \q(s) \angl[\big]{\piref(a | s) \bmupi(s,a), \bomega_0} + \gamma \angl*{ \q(s) \piref(a | s) \bmupi(s,a), \int d^{\pi}(\tilde{s}, \tilde{a}) \bphi(\tilde{s},\tilde{a}) \diff \tilde{s} \diff \tilde{a} } \\
    ={}& \biggl\langle \q(s) \piref(a | s) \bmupi(s,a), \underbrace{ (1-\gamma) \bomega_0 + \gamma \int d^{\pi}(\tilde{s}, \tilde{a}) \bphi(\tilde{s},\tilde{a}) \diff \tilde{s} \diff \tilde{a} }_{ \bomega_d^{\pi} } \biggr\rangle,
  \end{align*}
  where we use the initial representation (\Cref{assum:initial_distribution}) and the fact that $\pi(a | s) \bmu(s) = \piref(a | s) \bmupi(s,a)$. The representation of $\zeta(\cdot, \cdot)$ is hence a direct corollary since $q(s) \piref(a | s) = d^{\piref}(s,a)$ when $q(\cdot) = d^{\piref}(\cdot)$.
\end{proof}

Then, using the linear spectral representations of $Q$ and $\zeta$ in~\eqref{eq:linear_q} and~\eqref{eq:linear_zeta}, we shall equivalently formulate the DICE estimator as follows.

\begin{corollary}\label{thm:LP_formulation-representation}
  With primal-dual spectral representation \eqref{eq:spectral_representation_P_pi} where $q(\cdot) \equiv d^{\piref}(\cdot)$, under \Cref{assum:initial_distribution},
  \begin{align}\label{eq:LP_formulation-representation}
    \rhoP(\pi) &= \min_{\btheta_Q} \max_{\bomega_d} \biggl\lbrace (1-\gamma) \E[\begin{subarray}{l} s \sim \mu_0,\\ a \sim \pi(\cdot | s) \end{subarray}]{\bphi(s,a)^{\top} \btheta_Q} \\
    &\hspace{6em} + \mathbb{E}_{\begin{subarray}{l} s \sim d^{\piref}(\cdot),~ a \sim \piref(a | s),\\ s' \sim \P(\cdot | s,a),~ a' \sim \pi(\cdot | s') \end{subarray}} \Bigl[ \prn[\big]{\bmupi(s,a)^{\top} \bomega_d} \prn[\big]{r(s,a) + \gamma \bphi(s',a')^{\top} \btheta_Q - \bphi(s,a)^{\top} \btheta_Q} \Bigr] \biggr\rbrace. \nonumber
  \end{align}
\end{corollary}

The proof of \Cref{thm:LP_formulation-representation} is deferred to \Cref{sec:apdx-representation-proofs} due to limited space. We highlight that our new DICE formulation \eqref{eq:LP_formulation-representation} bears several benefits:
\begin{itemize}
  \item \textbf{Offline data compatible.} The estimator is favorable for OPE since the expectation over the $(s,a,s')$ transition pair can be effectively approximated by samples from the offline dataset $\D$, as long as the auxiliary distribution $\q(\cdot)$ is selected as the state occupancy measure $d^{\piref}$ of the behavior policy $\piref$ such that $\Prob{(s,a,s') \in \D} = q(s) \piref(a | s) \P(s' | s,a)$.
  
  \item \textbf{Optimization tractable.} Given (learned) $\bphi(s,a)$ and $\bmupi(s,a)$, the saddle-point optimization in \eqref{eq:LP_formulation-representation} is convex-concave with respect to both $\theta_Q$ and $\omega_d$, which perfectly bypasses the optimization difficulty in vanilla DICE estimators with neural-network-parameterized $Q^{\pi}(\cdot,\cdot)$ and $\zeta(\cdot,\cdot)$. Meanwhile, compared to the counterpart obtained by directly applying the naive spectral representation \eqref{eq:spectral_representation_P} (details of which can be found in \Cref{sec:apdx-representation-naive_failure}), the proposed estimator \eqref{eq:LP_formulation-representation} is tractable in that it is free of the policy ratio $\frac{\pi(a | s)}{\piref(a | s)}$ that is unknown.
\end{itemize}
From now on, we will always regard $q(\cdot) \equiv d^{\piref}(\cdot)$ for the aforementioned nice properties to hold.

\subsection{Spectral Representation Learning}

In the last section, we have elaborated on how to perform OPE using off-policy data given a primal-dual spectral representation. Now it only suffices to specify how to learn such a representation, which we regard as an abstract subroutine $(\hbphi, \hbmupi) \gets \replearn(\F, \D, \pi)$. Here $\F$ denotes the collection of candidate representations. We highlight that our algorithm works with any representation learning method that has a bounded learning error, without any further requirements on the learning mechanism. Given a range of spectral representation learning methods available in literature~\citep{zhang2022making, ren2022spectral, ren2022latent, shribak2024diffusion}, for the sake of clarity we only consider a few candidates here, while other methods may also be applicable:
\begin{enumerate}
  \item \textbf{Ordinary Least Squares (OLS).} Inspired by \citet{ren2022spectral}, an OLS objective can be constructed as follows. Denote by $\Qpi(s',a',s,a) := d^{\piref}(s,a) \Ppi(s',a' | s,a)$ the joint distribution of state-action transitions under behavior policy $\piref$, based on which we plug in \eqref{eq:spectral_representation_P_pi} to obtain
  \begin{equation*}
    \frac{\Qpi(s',a',s,a)}{\sqrt{d^{\piref}(s,a) d^{\piref}(s',a')}} = \sqrt{d^{\piref}(s,a) d^{\piref}(s',a')} \bphi(s,a)^{\top} \bmupi(s',a'),
  \end{equation*}
  which further induces the following OLS objective:
  \begin{equation*}
    \min_{(\hbphi, \hbmupi) \in \F}  \int \Biggl( \frac{\Qpi(s',a',s,a)}{\sqrt{d^{\piref}(s,a) d^{\piref}(s',a')}} - \sqrt{d^{\piref}(s,a) d^{\piref}(s',a')} \hbphi(s,a)^{\top} \hbmupi(s',a') \Biggr)^2 \diff s \diff a \diff s' \diff a'
  \end{equation*}
  Therefore, $(\hbphi, \hbmupi)$ can be learned by solving~\citep{ren2022spectral,haochen2021provable}:
  \begin{equation*}
    \min_{(\hbphi, \hbmupi) \in \F} \Bigl\lbrace \hE[(s,a) \sim d^{\piref}, (\tilde{s}', \tilde{a}') \sim d^{\piref}]{ \prn[\big]{ \hbphi(s,a)^{\top} \hbmupi(\tilde{s}', \tilde{a}') }^2} -2 \hE[(s,a) \sim d^{\piref}, (s',a') \sim \Ppi(\cdot,\cdot | s, a)]{\hbphi(s,a)^{\top} \hbmupi(s',a')} \Bigr\rbrace,
  \end{equation*}
  where the last term becomes a constant after expansion and is thus omitted. For practical implementation, we can use stochastic gradient descent to solve the above stochastic optimization problem.

  \item \textbf{Noise-Contrastive Estimation (NCE).} NCE is a widely used method for contrastive representation learning in RL \citep{zhang2022making, qiu2022contrastive}. To learn $(\hbphi, \hbmupi)$, we consider a binary contrastive learning objective \citep{qiu2022contrastive}:
  \begin{equationfit}
    \min_{(\hbphi, \hbmupi) \in \F} \widehat{E}_{(s,a) \sim d^{\piref}} \biggl[ \hE[(s',a') \sim \Ppi(\cdot,\cdot | s,a)]{\log \prn[\Big]{ 1 + \tfrac{1}{\hbphi(s,a)^{\top} \hbmupi(s',a')} } } + \hE[(s',a') \sim \Pneg]{ \log \prn[\Big]{1 + \hbphi(s,a)^{\top} \hbmupi(s',a')} } \biggr],
  \end{equationfit}
  where $\Pneg$ is a negative sampling distribution.
\end{enumerate}
Details of these representation learning methods along with their learning errors can be found in \Cref{sec:apdx-RepLearn}.

\subsection{\algname}

With the two key components specified above, now we are ready to state \algname, the proposed offline policy evaluation (OPE) algorithm using spectral representations, as displayed in \Cref{alg:main}.

\begin{algorithm*}[tb]
  \caption{\textbf{\algname}: \textbf{DI}stribution \textbf{C}orrection \textbf{E}stimation with \textbf{Spectral} Representation}\label{alg:main}
  \begin{algorithmic}[1]
    \Require Target policy $\pi$, off-policy dataset $\D$, function family $\F$.
    
    \State Learn a spectral representation $(\hbphi, \hbmupi) \gets \textbf{\replearn}(\F, \D, \pi)$.
  
    \State Plug in the spectral representation $(\hbphi, \hbmupi)$ to compute the following DICE estimator:
    \vspace{-3mm}
    \begin{align}\label{eq:alg-spectral_dice_estimator}
      \textstyle
      \hat{\rho}(\pi) = \min_{\btheta_Q} \max_{\bomega_d}& \bigg\lbrace (1-\gamma) \hE[\begin{subarray}{l} s \sim \mu_0,\\ a \sim \pi(\cdot | s) \end{subarray}]{\hbphi(s,a)^{\top} \btheta_Q} \nonumber\\
      &\quad + \hE[\begin{subarray}{c} (s,a,s') \sim \D,\\ a' \sim \pi(\cdot | s') \end{subarray}]{ \prn[\big]{\hbmu^{\pi}(s,a)^{\top} \bomega_d} \prn[\big]{r(s,a) + \gamma \hbphi(s',a')^{\top} \btheta_Q - \hbphi(s,a)^{\top} \btheta_Q}} \bigg\rbrace.
    \end{align}
    \vspace{-5mm}
    
    \State \Return $\hat{\rho}(\pi)$
  \end{algorithmic}
\end{algorithm*}

Specifically, given a policy $\pi$, assuming access to an offline dataset $(s,a,s') \sim \D$ sampled by the behavior policy $\piref$, we follow a two-step algorithm to evaluate the target policy $\pi$ in an off-policy manner:
\begin{enumerate}
  \item \textbf{Representation learning.} We may choose any representation learning method that comes with a bounded learning error as the \replearn subroutine, and the overall sample complexity will depend on this choice (see \Cref{sec:4-analysis}).

  \item \textbf{DICE-based policy evaluation.} With the learned representation $(\hbphi, \hbmupi)$, we use the primal-dual DICE estimator \eqref{eq:alg-spectral_dice_estimator} to estimate the value of the target policy $\pi$. Note that the data distribution $d^{\D}(s,a,s') = d^{\piref}(s) \piref(a | s) \P(s' | s,a)$ is exactly compatible with the formulation in \eqref{eq:LP_formulation-representation}.
\end{enumerate}

\begin{remark}[Numerical considerations]\label{remark:numerical}
  It is known that directly solving \eqref{eq:alg-spectral_dice_estimator} leads to potential numerical instability issues due to the objective's linearity in $\btheta_Q$ and $\bomega_d$ \citep{nachum2019algaedice}. Fortunately, it is shown in \citet{yang2020off} that certain regularization leads to strictly concave inner maximization while keeping the optimal \emph{solution} unbiased (see \Cref{sec:apdx-1_regularization} for details). In our implementation, we append a regularizer $- \lambda \hE[(s,a) \sim \D]{f(\hbmupi(s,a)^{\top} \bomega_d)}$ to the objective in \eqref{eq:alg-spectral_dice_estimator}, where $f$ is a differentiable function with closed and convex Fenchel conjugate $f_*$ (see \Cref{sec:apdx-4_Fenchel_conjugate}), and $\lambda$ is a tunable constant. Furthermore, since $\bmupi(s,a)^{\top} \bomega_d = \zeta(s,a) \leq \Cpi$ and $\bphi(s,a)^{\top} \btheta_Q = Q(s,a) \leq \frac{1}{1-\gamma}$, we also restrict $\btheta_Q$ and $\bomega_d$ in regions $\varTheta(\hbphi) = \set{\btheta_Q \mid 0 \leq \hbphi(s,a)^{\top} \btheta_Q \leq \frac{1}{1-\gamma}}$ and $\varOmega(\hbmupi) = \set{\bomega_d \mid \hbmupi(s,a)^{\top} \bomega_d \leq \Cpi}$, respectively.
\end{remark}
  \section{Theoretical Guarantee}\label{sec:4-analysis}

In this section, we provide a rigorously theoretical analysis regarding the sample complexity of the proposed \algname algorithm. For the sake of technical conciseness, we make the following assumption on the candidate family $\F$. We argue that this is not a restrictive assumption, but rather, only helps to highlight the key contributions with simplified analysis.

\begin{assumption}[realizability]\label{assum:well_defined_family}
  Assume a finite family $\F$, such that $\angl{ \hbphi(s,a), d^{\piref}(s',a') \hbmupi(s',a') }$ is a valid state-action transition kernel for any $(\hbphi, \hbmupi) \in \F$, and the ground-truth representation $(\sbphi, \sbmupi) \in \F$.
\end{assumption}

\paragraph{Representation Learning Error.} The key to subsequent analyses is to first bound the error of representation learning, which is of some theoretical interest by itself. Generally speaking, we expect \emph{probably approximately correct (PAC)} bounds for representation learning in the following format.

\begin{claim}\label{claim:representation_learning_error}
  With probability at least $1-\delta$, we have
  \begin{equation*}
    \E[(s,a) \sim d^{\piref}_{\P}]{\norm[\big]{\hP^{\pi}(\cdot, \cdot | s,a) - \Ppi(\cdot,\cdot | s,a)}_1} \leq \xi(|\F|, N, \delta),
  \end{equation*}
  where $\hPpi(s',a' | s,a) := d_{\P}^{\piref}(s', a') \hbphi(s,a)^{\top} \hbmu^{\pi}(s')$, $N$ is the number of samples in $\D$, and the upper bound $\xi$ only depends on $|\F|$, $N$ and $\delta$.
\end{claim}

We point out that, under certain regularity assumptions, the above claim can be proven for many spectral representation learning algorithms. Specifically, when \replearn is implemented by OLS or NCE, we can show that $\xi(|\F|, N, \delta) = \Theta\prn[\Big]{ \sqrt{\frac{1}{N} \log \frac{|\F|}{\delta}} }$.

\paragraph{Policy Evaluation Error.} The performance of the proposed \algname~algorithm is evaluated by the \emph{policy evaluation error} $\Err := \hat{\rho}(\pi) - \rhoP(\pi)$, which can be further bounded by the following theorem.

\begin{theorem}[Main Theorem]\label{thm:main_theorem}
  Suppose \Cref{claim:representation_learning_error} holds for the \replearn subroutine. Then under \Cref{assum:initial_distribution,assum:sufficient_sampling,assum:well_defined_family}, with probability at least $1-\delta$, we have
  \begin{equation*}
    \Err \lesssim \frac{1}{1-\gamma} \sqrt{\frac{\log(1 / \delta)}{N}} + \frac{1}{(1-\gamma)^2}  \cdot \xi(|\F|, N, \delta/2).
  \end{equation*}
\end{theorem}

\begin{proof}[Proof sketch]
We first split $\Err$ into the following terms:
\begin{equation*}
  \Err \;=\; 
  \underbrace{\hat{\rho}(\pi) - \bar{\rho}(\pi)}_{\textrm{statistical}}
  \;+\; \underbrace{\bar{\rho}(\pi) - \rho_{\hP}(\pi)}_{\textrm{dataset}}
  \;+\; \underbrace{\rho_{\hP}(\pi) - \rhoP(\pi)}_{\textrm{representation}},
\end{equation*}
where we introduce an auxiliary problem:
\begin{align*}
  \bar{\rho}(\pi) &= \min_{\btheta_Q} \max_{\bomega_d} \biggl\lbrace (1-\gamma) \E[\begin{subarray}{l} s \sim \mu_0,\\ a \sim \pi(\cdot | s) \end{subarray}]{\hbphi(s,a)^{\top} \btheta_Q} \\
  &\hspace{8em} + \mathbb{E}_{\begin{subarray}{l} s \sim d^{\piref}(\cdot),~ a \sim \piref(a | s),\\ (s',a') \sim \Ppi(\cdot,\cdot | s,a) \end{subarray}} \Bigl[ \prn[\big]{\hbmupi(s,a)^{\top} \bomega_d} \cdot{} \prn[\big]{r(s,a) + \gamma \hbphi(s',a')^{\top} \btheta_Q - \hbphi(s,a)^{\top} \btheta_Q} \Bigr] \biggr\rbrace. \nonumber
\end{align*}
Note that \eqref{eq:alg-spectral_dice_estimator} is the empirical estimation of $\bar{\rho}(\pi)$, and that $\bar{\rho}(\pi)$ is (subtly) inequivalent to $\rho_{\hP}(\pi)$---the expectation is still taken over $(s',a') \sim \Ppi(\cdot,\cdot | s,a)$ rather than $\hPpi(\cdot,\cdot | s,a) = \angl{\hbphi(s,a), \hbmupi(\cdot,\cdot)}$.

Intuitively, the latter two terms are directly related to the representation learning error established in \Cref{claim:representation_learning_error}, which can actually be bounded as follows:
\begin{align*}
  \rho_{\hP}(\pi) - \rhoP(\pi)
  &\lesssim \frac{\gamma}{(1-\gamma)^2} \cdot \xi(|\F|, N, \delta/2), \\
  \bar{\rho}(\pi) - \rho_{\hP}(\pi) &\lesssim \frac{1}{1 - \gamma} \cdot \xi(|\F|, N, \delta/2).
\end{align*}
On the other hand, the first term is only caused by replacing the expectations with their empirical estimators, which can be bounded by concentration inequalities as:
\begin{equation*}
  \hat{\rho}(\pi) - \bar{\rho}(\pi) \lesssim \frac{1}{1-\gamma} \sqrt{\frac{\log(1/\delta)}{N} }.
\end{equation*}
Plugging these terms back completes the proof. 
\end{proof}

Finally, we conclude that the sample complexity of \algname equipped with either OLS or NCE \replearn subroutine is $\tilde{O}(N^{-1/2})$ (under mild regularity assumptions). Details are deferred to \Cref{sec:apdx-analysis}.
  \section{Experiments}\label{sec:5-experiments}

In this section, we present experimental results in both continuous and discrete environments to demonstrate the strength of the proposed \algname algorithm. We also study the impact of hyperparameters, data coverage and the choice of behavior policy on the OPE performance, and illustrate the efficacy of the proposed representation learning method.

The empirical results show that our method outperforms \textsc{BestDICE}, the state-of-the-art DICE implementation without representation learning, in terms of both the convergence rate and the final prediction error. In comparison to other baselines, \algname achieves comparable performance with higher efficiency in simple environments, and performs significantly better than others in the most challenging environment.

\subsection{Continuous Environments}\label{sec:5-experiments-1-continuous}

\paragraph{Setting.} We start by comparing \algname with various baseline OPE methods in literature, including \textsc{BestDICE} \citep{yang2020off}, Fitted Q Evaluation (FQE) \citep{kostrikov2020statistical}, Model-Based (MB) method \citep{zhang2021autoregressive}, Importance Sampling (IS) method \citep{hanna2019importance} and Doubly-Robust (DR) method \citep{dudik2011doubly}. We follow the experiment protocol in \citet{yang2020off} to evaluate and compare the OPE performances of these algorithms in three continuous MuJoCo environments, namely \texttt{Cartpole}, \texttt{Reacher} and \texttt{Half-Cheetah}, in an increasing order of difficulty. In our implementation, for representation learning, we parameterize each of $\hbphi$ and $\hbmupi$ with a 2-layer feed-forward neural network. For the OPE step, regularizer is appended to \eqref{eq:alg-spectral_dice_estimator}, and the estimated policy value is retrieved by $\hat{\rho}(\pi) = \mathbb{E}_{(s,a) \sim d^{\mathcal{D}}} [ \hbmupi(s,a)^{\top} \bomega_d \cdot r(s,a) ]$ (see \Cref{remark:numerical}). 
Both steps are regarded as stochastic optimization problems, and are solved by stochastic gradient descent and stochastic gradient descent-ascent, respectively. Optimization hyperparameters are selected via grid search. Performance is quantified by \emph{OPE error} $|\hat{\rho}(\pi) - \rho(\pi)|$.

\begin{figure*}[t!]
    \centering
    \includegraphics[width=\textwidth]{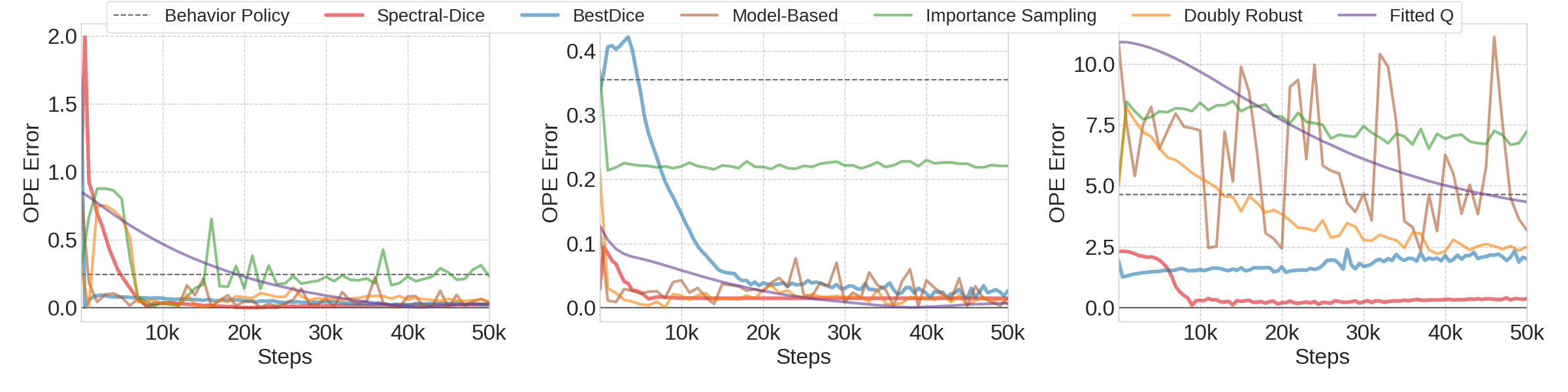}
    \vspace{-9mm}
    \caption{OPE error over the number of training steps in \texttt{Cartpole}, \texttt{Reacher} and \texttt{Half-Cheetah} environments (from left to right). Due to the use of convex-concave formulation, we can see that \algname converges faster and more stably to the target policy with a smaller OPE error in all three environments.}
    \label{fig:curve}
    \vspace{-4mm}
\end{figure*}

\begin{figure*}[t!]
\vspace{0.42cm}
\centering
\begin{minipage}[b]{.30\linewidth}
  \centering
  \vspace{-3mm}
  \includegraphics[width=.9\textwidth]{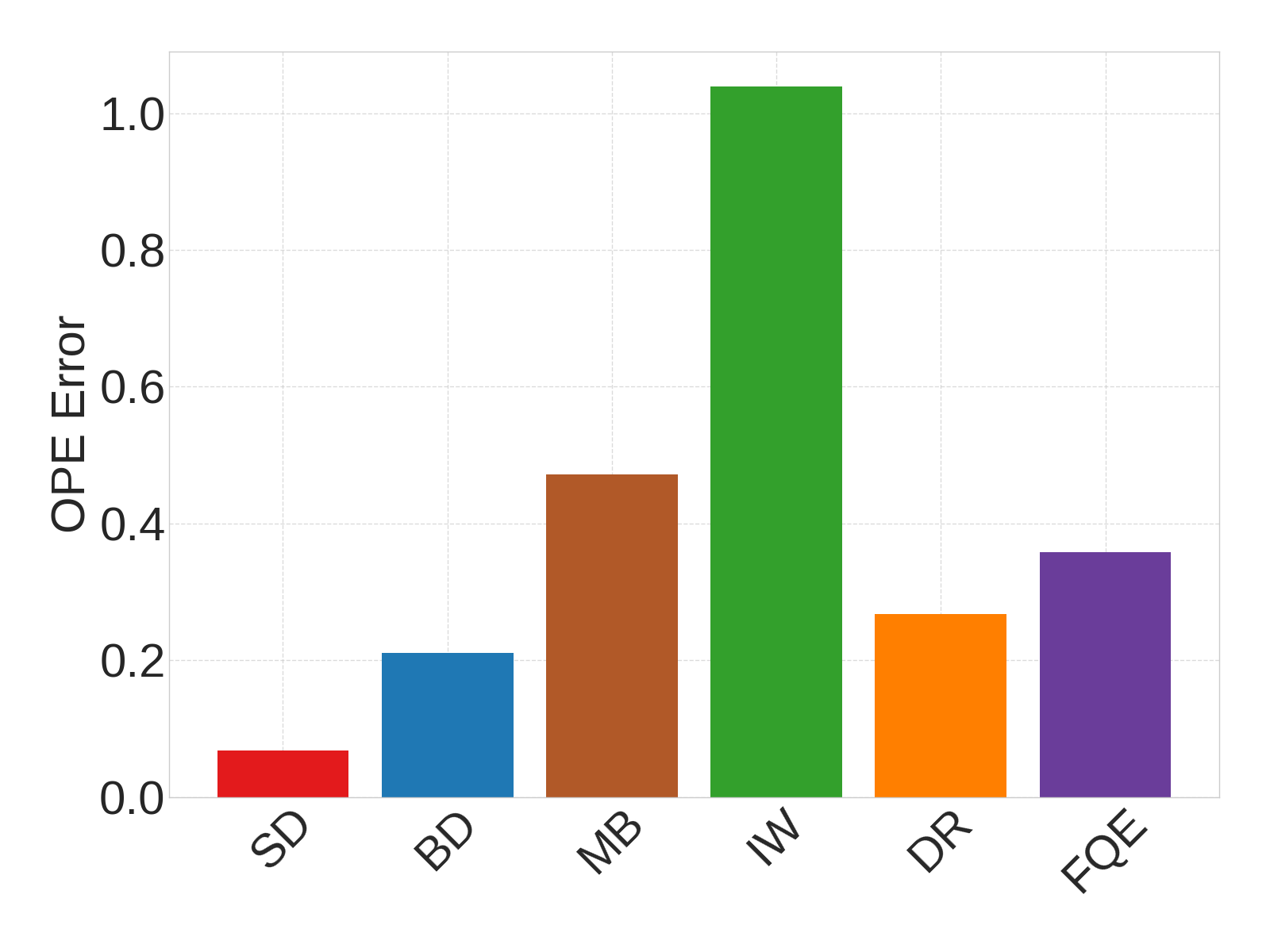}
  \label{fig:ablation study2}
  \vspace{-3mm}
  \captionof{figure}{Averaged relative OPE errors over three environments.}
  \label{fig:averaged error}
\end{minipage}
\hfill
\begin{minipage}[b]{.66\linewidth}
  \centering
  \begin{subfigure}[t]{0.49\linewidth}
    \includegraphics[width=\linewidth]{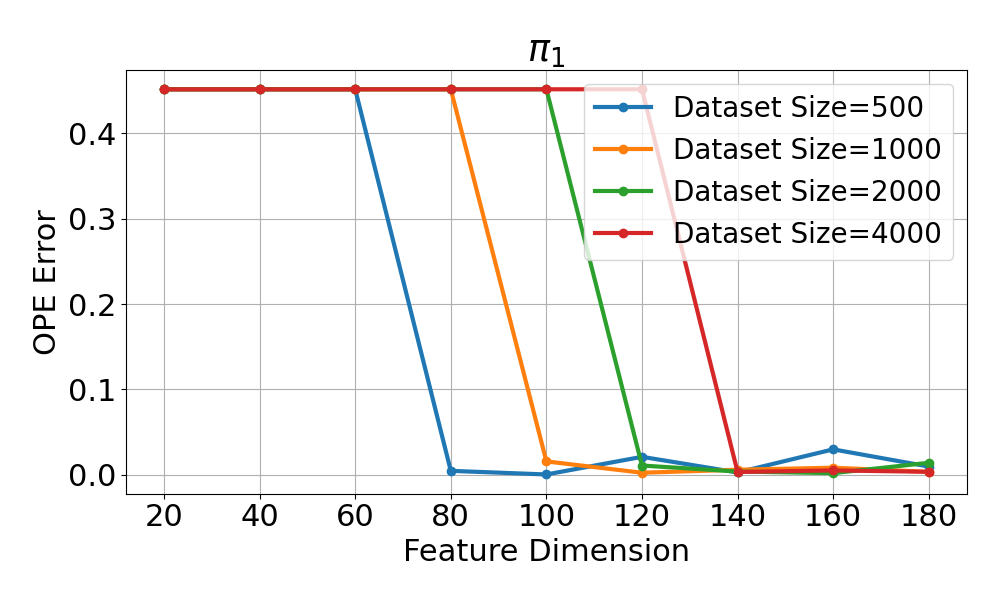}
    \label{fig:ablation study1}
  \end{subfigure}
  \begin{subfigure}[t]{0.49\linewidth}
    \includegraphics[width=\linewidth]{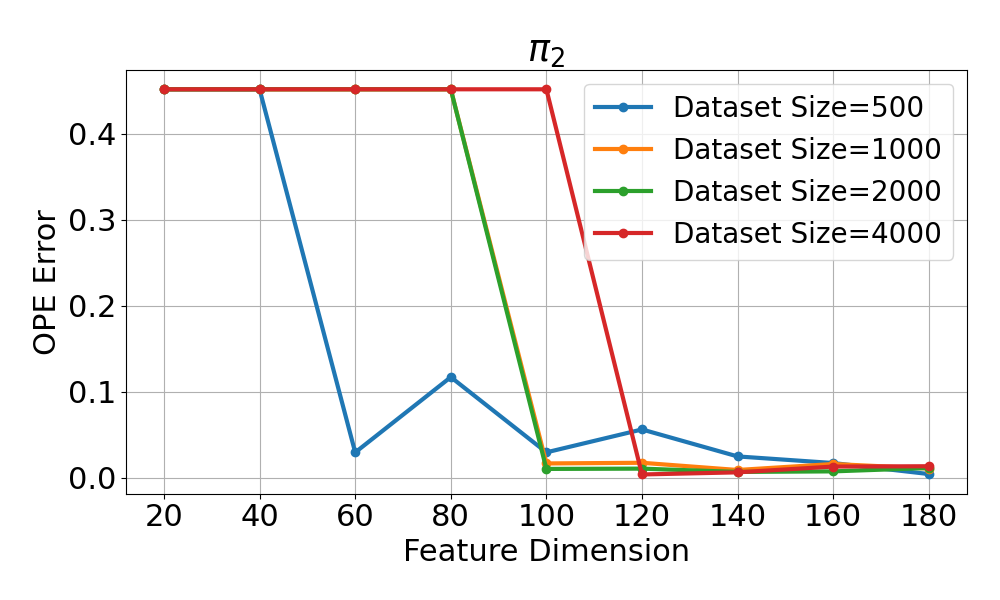}
  \end{subfigure}
  \vspace{-6mm}
  \captionof{figure}{OPE error of \algname in \texttt{Four Rooms} with varying behavior policies (``far-away'' policy $\pi_1$ vs. ``similar'' policy $\pi_2$), dataset sizes and feature dimensions.}
  \label{fig:ablation study}
\end{minipage}
\end{figure*}

\paragraph{Results.} The OPE performances of different methods in three environments are shown in \Cref{fig:curve}. It is observed that \algname achieves comparable performance in fewer optimization steps as compared to all the other baselines, and further, outperforms  them in terms of both convergence rate and final estimation error in the most challenging \texttt{Half-Cheetah} environment. Further, although FQE achieves an error close to \algname in simpler environments, its performance significantly degrades when the transition dynamics becomes more complex, demonstrating the importance and power of spectral representation.

Here we also highlight the comparison between two DICE-based methods---\algname (ours) and \textsc{BestDICE}. All settings showcase the advantage of our primal-dual spectral representation over the generic neural network representation, which justify the argument that, compared to the non-convex non-concave optimization in vanilla DICE, our convex-concave optimization leads to faster convergence and enhanced stability within a wider range of environments.

For a clearer comparison, we further present the averaged relative OPE error across these three environments in \Cref{fig:averaged error}. Here the \emph{relative OPE error} is defined by $\frac{|\hat{\rho}(\pi) - \rho(\pi)|}{|\bar{\rho}(\piref) - \rho(\pi)|}$, \ie, OPE error normalized by the value difference between the target and behavior policies. Under this metric, it becomes more evident that our method outperforms all the baselines in terms of estimation accuracy by a large margin.

\subsection{Discrete Environment}

\paragraph{Setting.}
We proceed to test our method in \texttt{Four Rooms} \citep{sutton1999between}, a classical discrete environment featuring convenient visualization, to study the algorithm's sensitivity for hyperparameters and illustrate the efficacy of representation learning. For representation learning in this tabular MDP, we perform singular value decomposition (SVD) of the matrix $\brak[\big]{ \frac{\Ppi(s',a' | s,a)}{d^{\pi}(s', a')} }$ (indexed by $(s,a)$ and $(s',a')$) and select the top $d$ singular vectors as $\hbphi(s,a)$ and $\hbmupi(s',a')$.

\paragraph{Sensitivity Study.} We study the algorithm's sensitivity with respect to behavior policy $\piref$, dataset size $N$ and spectral feature dimension $d$ by examining their impact on the OPE performance. For $\piref$, we vary between two behavior policies $\pi_1$ and $\pi_2$, where $\pi_1$ has a larger $\ell_1$-distance from the target policy than $\pi_2$. The results are shown in \Cref{fig:ablation study}. It can be observed that the proposed algorithm is always able to achieve low OPE errors with sufficiently large feature dimensions, showcasing its wide applicability under different behavior policies, data availability and hyperparameters.

\begin{figure}[htb]
  \centering
  \includegraphics[width=0.5\linewidth]{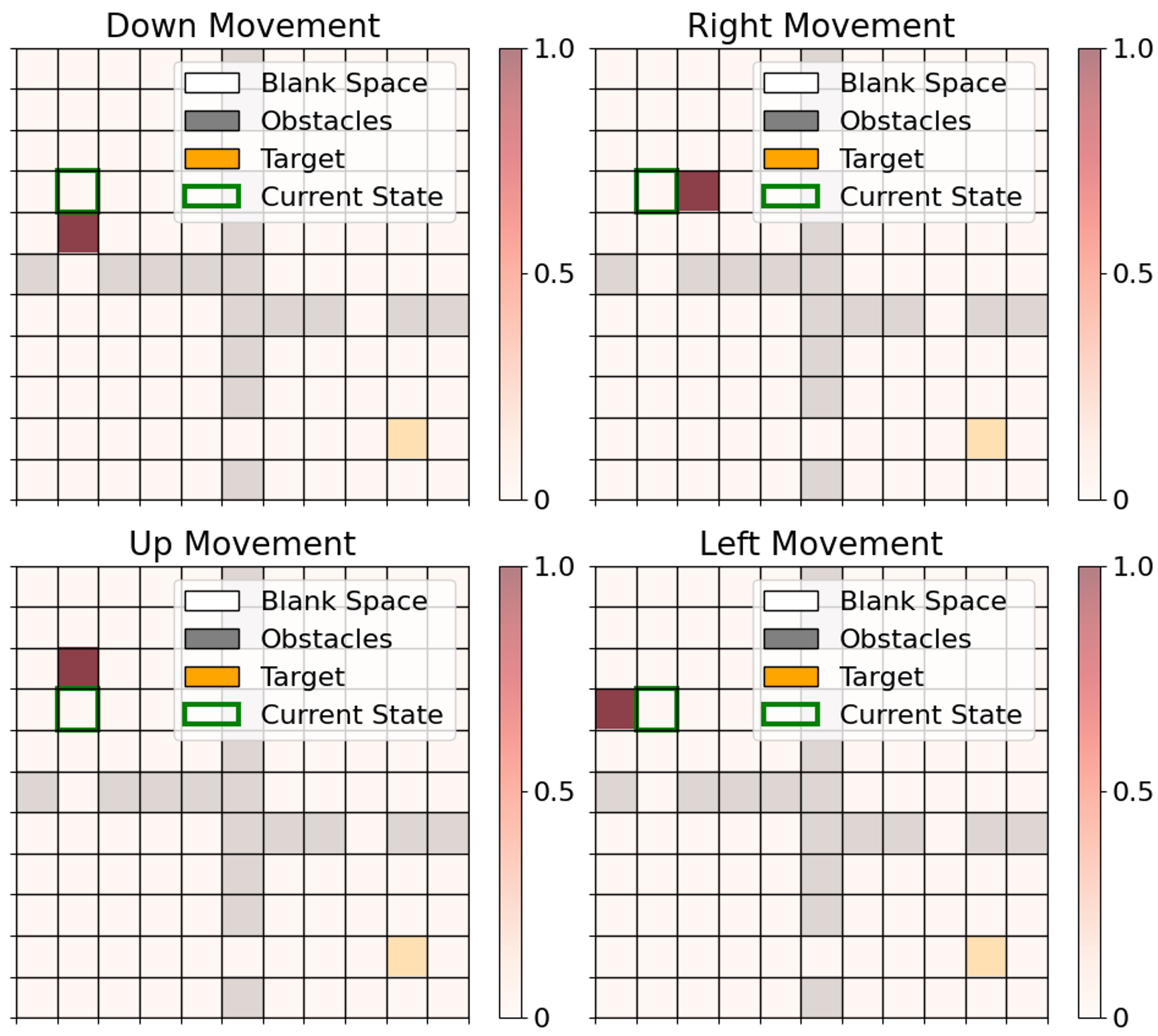}
  \vspace{-2mm}
  \caption{Visualization of the learned transition kernel for a fixed state and all the four actions.}
  \label{fig:transition_state}
\end{figure}

\paragraph{Efficacy of Representation Learning.}
To give a hint of the efficacy of our representation learning scheme \replearn, we visualize in \Cref{fig:transition_state} the learned transition kernel $\hP$ for a fixed state and all the four actions, where $\hP$ is restored from the spectral representation by \eqref{eq:spectral_representation_P}. As shown in the heat map (where darker color indicates higher probability), the \replearn algorithm successfully learns a set of primal-dual features that accurately encode the correct transition dynamics. 

More experimental details are deferred to~\Cref{sec:apdx-experiments}.
  \section{Conclusion}\label{sec:6-conclusion}

In this paper, to relieve the intrinsic tension between breaking the curse of horizon and overcoming the curse of dimensionality via DICE estimators, we propose a novel primal-dual spectral representation method that establishes linear spectral representations for both the primal variable (\ie, $Q$-function) and the dual variable (\ie, stationary distribution correction ratio), which leads to \algname,
an efficient and practical OPE algorithm that eliminates the non-convex non-concave saddle-point optimization in DICE and makes efficient use of historical data.
The performance of \algname is justified by a theoretical sample complexity guarantee and the empirical outperformance. Future directions include taking one step further to design offline policy optimization methods using primal-dual spectral representations, and applying the algorithm for efficient imitation learning.
   
  \bibliography{ref}
  \bibliographystyle{unsrtnat}


  \newpage
  \appendix

  \begin{center}
    \LARGE\textbf{\textsc{Appendix}}
  \end{center}
  \section{More Experimental Results}
\label{sec:apdx-experiments}

\paragraph{Additional Experiments.} We evaluate the OPE performance of the proposed \algname algorithm and the aforementioned baselines (see \Cref{sec:5-experiments-1-continuous}) in three additional environments, namely \texttt{Walker2d}, \texttt{Hopper} and \texttt{Ant}, the results of which are shown in \Cref{fig:appendix-curve}. These additional experiments further justify that our algorithm outperforms all the other baselines in a consistent and robust way, enjoying both a faster convergence rate and a smaller OPE error. These additional experimental results further confirm the superiority of \algname.

\begin{figure*}[ht!]
    \centering
    \includegraphics[width=0.98\textwidth]{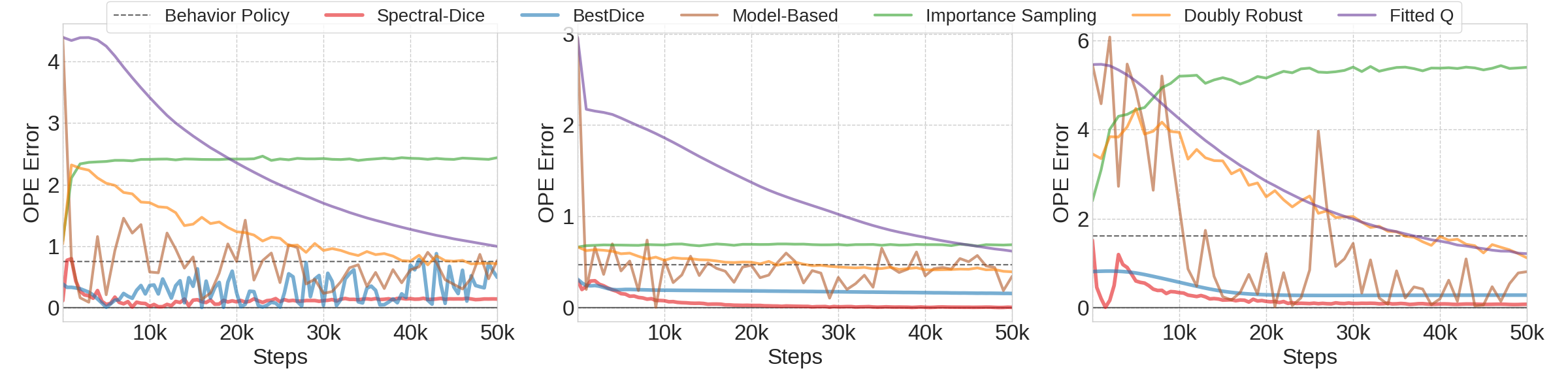}
    \vspace{-2mm}
    \caption{OPE error over the number of training steps in \texttt{Walker2d}, \texttt{Hopper} and \texttt{Ant} (from left to right). 
    }
    \label{fig:appendix-curve}
\end{figure*}

\paragraph{Implementation Details.} For the baseline algorithms, we follow the implementation of \textsc{BestDICE} in \cite{yang2020off} and the implementations of FQE, MB, IS, DR in \cite{fu2021benchmarks}. The optimization hyperparameters including learning rate, optimizer parameter, network architecture, batch size, \etc, are selected via grid search. All the experiments were conducted using V100 GPUs on a multi-node cluster.

For the continuous environments, the target policy is obtained using deep reinforcement learning agents (Deep Q-Network (DQN) agent for \texttt{Cartpole}, and Soft Actor-Critic (SAC) agents for all the other environments). The behavior policy is then obtained by sampling from a Gaussian distribution centered at the mean action of the target policy, where the variance of the Gaussian distribution can be adjusted to get behavior policies at different distances from the target policy. To build the offline dataset, we collect 400 trajectories using the behavior policy, where each trajectory is truncated to 250 steps. 

The source code is available at \url{\ttfamily https://anonymous.4open.science/r/spectral\_dice-720A}.

  \section{Primal-Dual Spectral Representation}\label{sec:apdx-representation}

In this appendix, we present the key properties of the proposed primal-dual spectral representation with proofs, as well as a brief discussion on why the spectral representation of that specific form is preferable.

\subsection{DICE Estimator with Primal-Dual Spectral Representation}\label{sec:apdx-representation-proofs}

We first present the proof of \Cref{thm:LP_formulation-representation} that is already stated in the main text.

\begin{restate}[\Cref{thm:LP_formulation-representation}]
  With primal-dual spectral representation \eqref{eq:spectral_representation_P_pi} where $q(\cdot) \equiv d^{\piref}(\cdot)$, under \Cref{assum:initial_distribution},
  \begin{align*}
    \rho(\pi) = \min_{\btheta_Q} \max_{\bomega_d}& \bigg\lbrace (1-\gamma) \E[\begin{subarray}{l} s \sim \mu_0,\\ a \sim \pi(\cdot | s) \end{subarray}]{\bphi(s,a)^{\top} \btheta_Q} \\
    &\quad + \E[\begin{subarray}{l} s \sim d^{\piref}(\cdot),~ a \sim \piref(a | s),\\ s' \sim \P(\cdot | s,a),~ a' \sim \pi(\cdot | s') \end{subarray}]{ \prn[\big]{\bmupi(s,a)^{\top} \bomega_d} \prn[\big]{r(s,a) + \gamma \bphi(s',a')^{\top} \btheta_Q - \bphi(s,a)^{\top} \btheta_Q} } \bigg\rbrace. \nonumber
  \end{align*}
\end{restate}

\begin{proof}[Proof of \Cref{thm:LP_formulation-representation}]
  Recall the primal-dual LP formulation of policy evaluation stated in \eqref{eq:LP_formulation-primal_dual}, which can be equivalently rewritten using the primal-dual spectral representation \eqref{eq:spectral_representation_P_pi} as follows:

  \vspace{-10pt}
  \begin{footnotesize}
  \begin{subequations}\label{eq:LP_formulation-representation-e1}
  \begin{align}
    \rho(\pi) &= \min_{Q(\cdot, \cdot)} \max_{d^{\pi}(\cdot,\cdot)} \brac*{ (1-\gamma) \E[\begin{subarray}{l} s \sim \mu_0,\\ a \sim \pi(\cdot | s) \end{subarray}]{Q(s,a)} + \int d^{\pi}(s,a) \brak*{r(s,a) + \gamma \E[\begin{subarray}{l} s' \sim \P(\cdot | s,a),\\ a' \sim \pi(\cdot | s') \end{subarray}]{Q(s',a')} - Q(s,a)} \diff s \diff a} \label{eq:LP_formulation-representation-e1:1}\\
    &= \min_{Q(\cdot, \cdot)} \max_{d^{\pi}(\cdot,\cdot)} \brac*{ (1-\gamma) \E[\begin{subarray}{l} s \sim \mu_0,\\ a \sim \pi(\cdot | s) \end{subarray}]{Q(s,a)} + \int \q(s) \piref(a | s) \cdot \frac{d^{\pi}(s,a)}{\q(s) \piref(a | s)}  \brak*{r(s,a) + \gamma \E[\begin{subarray}{l} s' \sim \P(\cdot | s,a),\\ a' \sim \pi(\cdot | s') \end{subarray}]{Q(s',a')} - Q(s,a)} \diff s \diff a} \label{eq:LP_formulation-representation-e1:2}\\
    &= \min_{Q(\cdot, \cdot)} \max_{d^{\pi}(\cdot,\cdot)} \brac*{ (1-\gamma) \E[\begin{subarray}{l} s \sim \mu_0,\\ a \sim \pi(\cdot | s) \end{subarray}]{Q(s,a)} + \E[\begin{subarray}{l} s \sim \q(\cdot),~ a \sim \piref(a | s),\\ s' \sim \P(\cdot | s,a),~ a' \sim \pi(\cdot | s') \end{subarray}]{ \frac{d^{\pi}(s,a)}{\q(s) \piref(\cdot | s)} \prn[\big]{r(s,a) + \gamma Q(s',a') - Q(s,a)} } } \label{eq:LP_formulation-representation-e1:3}\\
    &= \min_{\btheta_Q} \max_{\bomega_d} \brac*{ (1-\gamma) \E[\begin{subarray}{l} s \sim \mu_0,\\ a \sim \pi(\cdot | s) \end{subarray}]{\bphi(s,a)^{\top} \btheta_Q} + \E[\begin{subarray}{l} s \sim d^{\piref}(\cdot),~ a \sim \piref(\cdot | s),\\ s' \sim \P(\cdot | s,a),~ a' \sim \pi(\cdot | s') \end{subarray}]{ \bmupi(s,a)^{\top} \bomega_d \prn[\big]{r(s,a) + \gamma \bphi(s',a')^{\top} \btheta_Q - \bphi(s,a)^{\top} \btheta_Q} } }, \label{eq:LP_formulation-representation-e1:4}
  \end{align}
  \end{subequations}
  \end{footnotesize}
  \vspace{-6pt}
  
  \noindent where in \eqref{eq:LP_formulation-representation-e1:2} we perform the IS-style change-of-variable used in DICE estimators (see \eqref{eq:IS_change_of_variable}); in \eqref{eq:LP_formulation-representation-e1:4} we plug in the primal-dual spectral representation of $Q^{\pi}$ and $d^{\pi}$ stated in \eqref{eq:spectral_representation_P_pi}, as well as the fact that $q(\cdot) \equiv d^{\piref}(\cdot)$.
\end{proof}

\subsection{Failure of the Naive Spectral Representation}\label{sec:apdx-representation-naive_failure}

In \Cref{sec:3-algorithm-1-representation}, it is mentioned that directly applying the naive spectral representation \eqref{eq:spectral_representation_P} proposed in \citet{ren2022spectral} induces a complicated representation for $\zeta(\cdot,\cdot)$, which in turn leads to an intractable optimization \eqref{eq:LP_formulation-primal_dual} for the computation of the DICE estimator. The above point is further elaborated here in a formal way.

Note that, in \Cref{thm:linear_representation_Q_d}, the linear representation of $Q^{\pi}$ only builds upon the low-rank MDP assumption, and therefore it still holds with the naive spectral representation \eqref{eq:spectral_representation_P}. Meanwhile, it can be checked that
\begin{equation}
  d^{\pi}(s,a)
  = \biggl\langle \q(s) \pi(a | s) \bmu(s), \underbrace{ (1-\gamma) \bomega_0 + \gamma \int d^{\pi}(\tilde{s}, \tilde{a}) \bphi(\tilde{s},\tilde{a}) \diff \tilde{s} \diff \tilde{a} }_{ \bomega_d^{\pi} } \biggr\rangle,
\end{equation}
which can be obtained by plugging the relation $\pi(a | s) \bmu(s) = \piref(a | s) \bmupi(s,a)$ into the linear representation of $d^{\pi}(\cdot, \cdot)$ to eliminate $\bmupi$ from the representation. Consequently, the LP formulation \eqref{eq:LP_formulation-representation-e1} becomes

\vspace{-8pt}
\begin{footnotesize}
\begin{align*}
  \rho(\pi) &= \min_{Q(\cdot, \cdot)} \max_{d^{\pi}(\cdot)} \brac*{ (1-\gamma) \E[\begin{subarray}{l} s \sim \mu_0,\\ a \sim \pi(\cdot | s) \end{subarray}]{Q(s,a)} + \int d^{\pi}(s,a) \brak*{r(s,a) + \gamma \E[\begin{subarray}{l} s' \sim \P(\cdot | s,a),\\ a' \sim \pi(\cdot | s') \end{subarray}]{Q(s',a')} - Q(s,a)} \diff s \diff a} \\
  &= \min_{Q(\cdot, \cdot)} \max_{d^{\pi}(\cdot)} \brac*{ (1-\gamma) \E[\begin{subarray}{l} s \sim \mu_0,\\ a \sim \pi(\cdot | s) \end{subarray}]{Q(s,a)} + \int \q(s) \piref(a | s) \cdot \textcolor{red}{ \frac{\pi(a | s)}{\piref(a | s)} } \frac{d^{\pi}(s,a)}{\q(s) \pi(a | s)} \brak*{r(s,a) + \gamma \E[\begin{subarray}{l} s' \sim \P(\cdot | s,a),\\ a' \sim \pi(\cdot | s') \end{subarray}]{Q(s',a')} - Q(s,a)} \diff s \diff a} \\
  &= \min_{Q(\cdot, \cdot)} \max_{d^{\pi}(\cdot)} \brac*{ (1-\gamma) \E[\begin{subarray}{l} s \sim \mu_0,\\ a \sim \pi(\cdot | s) \end{subarray}]{Q(s,a)} + \E[\begin{subarray}{l} s \sim \q(\cdot),~ a \sim \piref(\cdot | s),\\ s' \sim \P(\cdot | s,a),~ a' \sim \pi(\cdot | s') \end{subarray}]{ \textcolor{red}{ \frac{\pi(a | s)}{\piref(a | s)} } \frac{d^{\pi}(s,a)}{\q(s) \pi(a | s)} \prn[\big]{r(s,a) + \gamma Q(s',a') - Q(s,a)} } } \\
  &= \min_{\btheta_Q} \max_{\bomega_d} \brac*{ (1-\gamma) \E[\begin{subarray}{l} s \sim \mu_0,\\ a \sim \pi(\cdot | s) \end{subarray}]{\bphi(s,a)^{\top} \btheta_Q} + \E[\begin{subarray}{l} s \sim \q(\cdot),~ a \sim \piref(\cdot | s),\\ s' \sim \P(\cdot | s,a),~ a' \sim \pi(\cdot | s') \end{subarray}]{  \textcolor{red}{ \frac{\pi(a | s)}{\piref(a | s)} } \prn[\big]{\bmu(s)^{\top} \bomega_d} \prn[\big]{r(s,a) + \gamma \bphi(s',a')^{\top} \btheta_Q - \bphi(s,a)^{\top} \btheta_Q} } },
\end{align*}
\end{footnotesize}

\noindent which involves an unknown ratio $\frac{\pi(a | s)}{\piref(a | s)}$ when the behavior policy $\piref$ is unknown, and is thus intractable.

The above failed attempt implies that the policy ratio should be ``absorbed'' into the representation to be implicitly learned during representation learning, which exactly inspires the primal-dual spectral representation \eqref{eq:spectral_representation_P_pi}.

\subsection{Solving the Minimax Problem via Regularization}\label{sec:apdx-1_regularization}

It is known that directly solving \eqref{eq:alg-spectral_dice_estimator} leads to potential numerical instability issues due to the objective's linearity in $\btheta_Q$ and $\bomega_d$ \citep{nachum2019algaedice}. Fortunately, it is shown in \citet{yang2020off} that certain regularization leads to strictly concave inner maximization while keeping the optimal \emph{solution} $\bomega_d^{\star}$ unbiased. Specifically, in practical implementation we may append the following regularizer to the objective in \eqref{eq:LP_formulation-representation}:
\begin{align}\label{eq:LP_formulation-regularized}
  \rho_{\reg}(\pi) &= \min_{\btheta_Q} \max_{\bomega_d} \biggl\lbrace (1-\gamma) \E[\begin{subarray}{l} s \sim \mu_0,\\ a \sim \pi(\cdot | s) \end{subarray}]{\bphi(s,a)^{\top} \btheta_Q} + \mathbb{E}_{\begin{subarray}{l} s \sim d^{\piref}(\cdot),~ a \sim \piref(a | s),\\ s' \sim \P(\cdot | s,a),~ a' \sim \pi(\cdot | s') \end{subarray}} \Bigl[ \prn[\big]{\bmupi(s,a)^{\top} \bomega_d} \cdot{} \nonumber\\
  &\hspace{7em} \prn[\big]{r(s,a) + \gamma \bphi(s',a')^{\top} \btheta_Q - \bphi(s,a)^{\top} \btheta_Q} \Bigr]  \textcolor{red}{{}- \lambda \E[(s,a) \sim \D]{f(\hbmupi(s,a)^{\top} \bomega_d)} } \bigg\rbrace.
\end{align}
Here $f$ is a differentiable convex function with closed and convex Fenchel conjugate $f_*$ (see \Cref{sec:apdx-4_Fenchel_conjugate}), and $\lambda > 0$ is a tunable constant that controls the magnitude of regularization. It is evident that the regularized objective is concave in $\bomega_d$, which facilitates the inner maximization. What's more, it has also be shown that such regularization does not alter the optimal solution $\bomega_d^{\star}$, as summarized below.

\begin{lemma}[\citet{nachum2019algaedice, yang2020off}]
  The solution $(\btheta_Q^{\reg,\star}, \bomega_d^{\reg,\star})$ to \eqref{eq:LP_formulation-regularized} satisfies:
  \begin{align*}
    \bphi(s,a)^{\top} \btheta_Q^{\reg,\star} &= \bphi(s,a)^{\top} \btheta_Q^{\star} - \lambda (\mathcal{I} - \mathcal{P}^{\pi})^{-1} f'\prn*{\tfrac{d^{\pi}(s,a)}{d^{\piref}(s,a)}},\\
    \bmupi(s,a)^{\top} \bomega_d^{\star} &= \bmupi(s,a)^{\top} \bomega_d^{\reg,\star}, \\
    \rho_{\reg}(\pi) &= \rho(\pi) - \lambda \Div_f \prn{d^{\pi} \Vert d^{\piref}},
  \end{align*}
  where $(\btheta_Q^{\star}, \bomega_d^{\star})$ is the solution to \eqref{eq:LP_formulation-representation}.
\end{lemma}

We emphasize that the regularized problem is unbiased only in the sense that $\bomega_d^{\reg,\star} = \bomega_d^{\star}$. Therefore, in general we need to plug $\bomega_d^{\reg,\star}$ back into \eqref{eq:LP_formulation-representation} and solve the outer minimization again to recover $\btheta_Q^{\star}$. Nevertheless, when $\lambda$ is sufficiently small, we shall regard $\btheta_Q^{\reg,\star} \approx \btheta_Q^{\star}$ to relieve the computational burden.

In practice, we can only solve the empirical version of \eqref{eq:LP_formulation-regularized}, \ie,
\begin{align*}
  \rho_{\reg}(\pi) &= \min_{\btheta_Q} \max_{\bomega_d} \biggl\lbrace (1-\gamma) \hE[\begin{subarray}{l} s \sim \mu_0,\\ a \sim \pi(\cdot | s) \end{subarray}]{\bphi(s,a)^{\top} \btheta_Q} + \widehat{\mathbb{E}}_{\begin{subarray}{l} s \sim \mu_0,\\ a \sim \pi(\cdot | s) \end{subarray}} \Bigl[ \prn[\big]{\bmupi(s,a)^{\top} \bomega_d} \cdot{} \\
  &\hspace{7em} \prn[\big]{r(s,a) + \gamma \bphi(s',a')^{\top} \btheta_Q - \bphi(s,a)^{\top} \btheta_Q} \Bigr]  \textcolor{red}{{}- \lambda \hE[(s,a) \sim \D]{f(\hbmupi(s,a)^{\top} \bomega_d)} } \bigg\rbrace.
\end{align*}
  \section{Representation Learning Methods and Their Error Bounds}\label{sec:apdx-RepLearn}

In this appendix, we introduce two candidate methods---\emph{ordinary least squares (OLS)} and \emph{noise-contrastive estimation (NCE)}---that can be used as the \replearn subroutine. Further, we also provide their representation learning error bounds in the form of \Cref{claim:representation_learning_error}, which is restated here for readers' convenience:

\begin{restate}[\Cref{claim:representation_learning_error}]
  With probability at least $1-\delta$, the representation learning error of $\replearn(\F, \D, \pi)$ is bounded by
  \begin{equation*}
    \E[(s,a) \sim d^{\piref}_{\P}]{\norm[\big]{\hP^{\pi}(\cdot, \cdot | s,a) - \Ppi(\cdot,\cdot | s,a)}_1} \leq \xi(|\F|, N, \delta),
  \end{equation*}
  where $\hPpi(s',a' | s,a) := d^{\piref}_{\P}(s',a') \hbphi(s,a)^{\top} \hbmu^{\pi}(s')$, and $N$ is the number of samples in $\D$.
\end{restate}

It should be emphasized that the two methods discussed here are not the only candidates for \replearn. Rather, any representation learning method that comes with a learning error bound in the required form is applicable, without any further requirements on the learning mechanism.

\subsection{Ordinary Least Sqaures (OLS)}

\paragraph{Method.} Inspired by \citet{ren2022spectral}, the objective of OLS can be constructed as follows. Denote by $\Qpi(s',a',s,a) := d^{\piref}(s,a) \Ppi(s',a' | s,a)$ the joint distribution of state-action transitions under behavior policy $\piref$. Then we plug $\Qpi$ into \eqref{eq:spectral_representation_P_pi} and rearrange the terms to obtain
\begin{equation*}
  \frac{\Qpi(s',a',s,a)}{\sqrt{d^{\piref}(s,a) d^{\piref}(s',a')}} = \sqrt{d^{\piref}(s,a) d^{\piref}(s',a')} \bphi(s,a)^{\top} \bmupi(s',a').
\end{equation*}
Therefore, we propose to optimize over the following OLS objective:
\begin{align*}
  &\min_{(\hbphi, \hbmupi) \in \F}  \int \Biggl( \frac{\Qpi(s',a',s,a)}{\sqrt{d^{\piref}(s,a) d^{\piref}(s',a')}} - \sqrt{d^{\piref}(s,a) d^{\piref}(s',a')} \hbphi(s,a)^{\top} \hbmupi(s',a') \Biggr)^2 \diff s \diff a \diff s' \diff a' \\
  ={}& \min_{(\hbphi, \hbmupi) \in \F} \biggl\lbrace
  \int \frac{\Qpi(s',a',s,a)^2}{d^{\piref}(s, a) d^{\piref}(s', a')} \diff s \diff a \diff s' \diff a' - 2 \E[(s,a) \sim d^{\piref}, (s',a') \sim \Ppi(\cdot,\cdot | s, a)]{\hbphi(s,a)^{\top} \hbmupi(s',a')} \\
  &\hspace{5em} + \E[(s,a) \sim d^{\piref}, (s', a') \sim d^{\piref}]{ \prn[\big]{ \hbphi(s,a)^{\top} \hbmupi(s', a') }^2} \biggr\rbrace,
\end{align*}
Note that the first term $\int \frac{\Qpi(s',a',s,a)^2}{d^{\piref}(s, a) d^{\piref}(s', a')} \diff s \diff a \diff s' \diff a'$ is a constant that can be omitted in optimization, while the second and third terms can be effectively approximated by sampling from the dataset $\D$ and the target policy $\pi$. Therefore, in practice we learn $(\hbphi, \hbmupi)$ by solving the following optimization:
\begin{equation}\label{eq:OLS_objective}
  \min_{(\hbphi, \hbmupi) \in \F} \brac[\bigg]{ \hE[(s,a) \sim d^{\piref}, (\tilde{s}', \tilde{a}') \sim d^{\piref}]{ \prn[\big]{ \hbphi(s,a)^{\top} \hbmupi(\tilde{s}', \tilde{a}') }^2} -2 \hE[(s,a) \sim d^{\piref}, (s',a') \sim \Ppi(\cdot,\cdot | s, a)]{\hbphi(s,a)^{\top} \hbmupi(s',a')} },
\end{equation}
where the expectations are replaced by their empirical estimations using data sampled from $\D$.

\medskip
\paragraph{Error Bound.} We proceed to show the representation learning error bound for the OLS method, which requires the following regularity assumption on the transition kernel $\Ppi$ and the occupancy measure $d^{\piref}$. 

\begin{assumption}[regularity for OLS]\label{assum:OLS_regularity}
  (1) lower-bounded transition kernel: $\Ppi(s',a' | s,a) \geq \frac{1}{\CP} > 0$, $\forall s,a,s',a'$; (2) effective behavior policy coverage: $\frac{d^{\piref}(s,a)}{d^{\piref}(s',a')} \leq \Ccov$, $\forall s,a,s',a'$.
\end{assumption}

We point out that the major rationale behind these mild assumptions is to rule out the cases where certain transitions are scarcely sampled due to the singularity in transition kernel or behavior policy.

\begin{theorem}[OLS learning error]\label{thm:OLS_learning_error}
  Under \Cref{assum:initial_distribution,assum:sufficient_sampling,assum:well_defined_family} and the additional \Cref{assum:OLS_regularity} for regularity, let $(\hbphi, \hbmupi)$ be the solution to \eqref{eq:OLS_objective}, and set $\hPpi(s',a' | s,a) := d^{\piref}(s', a') \hbphi(s,a)^{\top} \hbmu^{\pi}(s')$. Then, for any $\delta \in (0,1)$, with probability at least $1-\delta$, we have
  \begin{equation*}
    \E[(s,a) \sim d_{\P}^{\piref}]{ \norm[\big]{\Ppi(\cdot,\cdot | s,a) - \hPpi(\cdot,\cdot | s,a)}_1 }
    \leq \sqrt{\CP \Creg} \cdot \sqrt{\frac{\log(|\F|/\delta)}{N}},
  \end{equation*}
  where $\Creg = \frac{4}{3}\sqrt{\Ccov} + 8 \Ccov$ is a universal constant determined by the PAC bound for OLS..
\end{theorem}

\begin{proof}
  We would like to apply the fast-rate PAC bound for OLS regression (\Cref{thm:PAC_regression_fast_rate}). For the sake of clarity, we explicitly define the family of candidate regression functions as
  \begin{equation*}
    \tF := \set*{f: (s,a,s',a') \mapsto \sqrt{d^{\piref}(s,a) d^{\piref}(s',a')} \bphi(s,a)^{\top} \bmupi(s',a') \;\middle|\; (\bphi,\bmupi) \in \F}.
  \end{equation*}
  It is evident that any $f \in \tF$ is bounded as follows: 
  \begin{equation*}
    0 \leq f(s,a,s',a')
    = \sqrt{\frac{d^{\piref}(s,a)}{d^{\piref}(s',a')}} \tilde{\P}^{\pi}(s',a' | s,a)
    \leq \sqrt{\Ccov},
  \end{equation*}
  where we use the fact that $\angl{ \hbphi(s,a), d^{\piref}(s',a') \hbmupi(s',a') }$ is always some valid transition kernel $\tilde{\P}^{\pi}$ (by \Cref{assum:well_defined_family}), and the additional regularity assumption (\Cref{assum:OLS_regularity}). Further, since the family $\tF$ is realizable (by \Cref{assum:well_defined_family}), there exists an optimal $\fs \in \tF$ such that
  \begin{equation*}
    \fs(s,a,s',a') = \frac{\Qpi(s',a',s,a)}{\sqrt{d^{\piref}(s,a) d^{\piref}(s',a')}}.
  \end{equation*}
  As $f(s,a,s',a'), \fs(s,a,s',a') \in [0, \sqrt{\Ccov}]$, we deduce from \Cref{thm:PAC_regression_fast_rate} that, with probability at least $1-\delta$,
  \begin{align}\label{eq:OLS_replearn_pac}
    \int \prn*{ \fs(s,a,s',a') - \hf(s,a,s',a') }^2 \diff s \diff a \diff s' \diff a'
    \leq \Creg \cdot \frac{\log(|\F|/\delta)}{N},
  \end{align}
  where $\Creg := \frac{4}{3}\sqrt{\Ccov} + 8 \Ccov$, and $\hf(s,a,s',a') := \sqrt{d^{\piref}(s,a) d^{\piref}(s',a')} \hbphi(s,a)^{\top} \hbmupi(s',a')$. Consequently,
  \begin{subequations}\label{eq:OLS_replearn_error}
  \begin{align}
    &\E[(s,a) \sim d_{\P}^{\piref}]{ \norm[\big]{\Ppi(\cdot,\cdot | s,a) - \hPpi(\cdot,\cdot | s,a)}_1 } \nonumber\\
    ={}& \int d_{\P}^{\piref}(s,a) \abs*{\Ppi(s',a' | s,a) - \hPpi(s',a' | s,a)} \diff s \diff a \diff s' \diff a' \label{eq:OLS_replearn_error:1}\\
    ={}& \int \abs*{\Qpi(s',a',s,a) - \hQpi(s',a',s,a)} \diff s \diff a \diff s' \diff a' \label{eq:OLS_replearn_error:2}\\
    \leq{}& \sqrt{ \int \prn*{ \sqrt{\Qpi(s',a',s,a)} - \frac{\hQpi(s',a',s,a)}{\sqrt{\Qpi(s',a',s,a)}} }^2 \diff s \diff a \diff s' \diff a' \cdot \int \Qpi(s',a',s',a') \diff s \diff a \diff s' \diff a' } \label{eq:OLS_replearn_error:3}\\
    ={}& \sqrt{ \int \frac{d^{\piref}(s',a')}{\Ppi(s',a' | s,a)} \prn*{ \fs(s,a,s',a') - \hf(s,a,s',a') }^2 \diff s \diff a \diff s' \diff a' } \label{eq:OLS_replearn_error:4}\\
    \leq{}& \sqrt{ \max_{s,a,s',a'} \brac*{\frac{d^{\piref}(s',a')}{\Ppi(s',a' | s,a)}} } \cdot \sqrt{\Creg \cdot \frac{\log(|\F|/\delta)}{N}} \label{eq:OLS_replearn_error:5}\\
    \leq{}& \sqrt{\CP \Creg} \cdot \sqrt{\frac{\log(|\F|/\delta)}{N}}, \label{eq:OLS_replearn_error:6}
  \end{align}
  \end{subequations}
  where in \eqref{eq:OLS_replearn_error:2} we use the definition of $\Qpi$, and define $\hQpi := d^{\piref}_{\P}(s,a) \hPpi(s',a' | s,a)$; in \eqref{eq:OLS_replearn_error:3} we use Cauchy-Schwartz inequality; in \eqref{eq:OLS_replearn_error:4} we use the definition of $\hf$ and $\fs$; in \eqref{eq:OLS_replearn_error:5} we plug in the PAC bound \eqref{eq:OLS_replearn_pac}; in \eqref{eq:OLS_replearn_error:6} we use \Cref{assum:OLS_regularity} to bound the coefficient. This completes the proof.
\end{proof}

\subsection{Noise-Constrastive Learning (NCE)}

\paragraph{Method.} NCE is a widely used method for contrastive representation learning in RL \citep{zhang2022making, qiu2022contrastive}. To learn $(\hbphi, \hbmupi)$, we consider a binary contrastive learning objective \citep{qiu2022contrastive}:
\begin{equation}\label{eq:NCE_objective}
  \min_{(\hbphi, \hbmupi) \in \F} \hE[(s,a) \sim d^{\piref}]{ \hE[(s',a') \sim \Ppi(\cdot,\cdot | s,a)]{\log \prn[\Big]{ 1 + \tfrac{1}{\hbphi(s,a)^{\top} \hbmupi(s',a')} } } + \hE[(s',a') \sim \Pneg]{ \log \prn[\Big]{1 + \hbphi(s,a)^{\top} \hbmupi(s',a')} } },
\end{equation}
where $\Pneg$ is a negative sampling distribution that will be specified with justification later. We highlight that the above objective implicitly guarantees an equal number of positive and negative samples.

The following derivations follow a similar pathway as those in \cite{qiu2022contrastive}. For notational consistency that facilitates the application of known results, we introduce the following auxiliary notations. Define
\begin{equation*}
  \tF := \set*{f: (s,a,s',a') \mapsto \bphi(s,a)^{\top} \bmupi(s',a') \mid (\bphi,\bmupi) \in \F}.
\end{equation*}
For clarity, we augment the sampled transitions to include a label $y$ indicating whether the sample is positive ($y=1$) or negative ($y=0$). Formally, given a dataset $\D = \set{(s_i, a_i, s'_i, a'_i) \mid i \in [N]}$ of positive transitions, we randomly sample $N$ negative transitions $(\tilde{s}_i, \tilde{a}_i) \sim \Pneg$ ($i \in [N]$, i.i.d.), and define the augmented dataset
\begin{equation*}
  \tD := \set[\big]{ (s_i, a_i, s'_i, a'_i, 1),(s_i, a_i, \tilde{s}_i, \tilde{a}_i, 0) \;\big|\; i \in [N] }.
\end{equation*}
In this way, the NCE objective \eqref{eq:NCE_objective} can be equivalently rewritten (in MLE format) as
\begin{equation}\label{eq:NCE_objective_equiv}
  \max_{f \in \tF}\; \hE[(s,a,s',a',y) \sim d^{\tD}]{\log \psi_f(s,a,s',a',y)},
\end{equation}
where the likelihood function $\psi_f$ is defined by
\begin{equation*}
  \psi_f(s,a,s',a',y) := \prn*{ \frac{f(s,a,s',a')}{1+f(s,a,s',a')} }^y \cdot \prn*{ \frac{1}{1+f(s,a,s',a')} }^{1-y}.
\end{equation*}
We point out that $\psi_f(s,a,s',a',\cdot) \in \Delta(\Y)$ for any $(s,a,s',a')$, where $\Y := \set{0,1}$. In fact, given $\fs$ that optimizes the unconstrained non-empirical version of \eqref{eq:NCE_objective_equiv}, $\psi_{\fs}$ can be interpreted as the probability of obtaining label $y$ given $(s,a,s',a')$, as summarized in the following lemma that is similar to Lemma C.1 in \citet{qiu2022contrastive}.

\begin{lemma}[non-empirical solution to NCE]\label{thm:NCE_solution}
  The optimal solution to the unconstrained non-empirical version of \eqref{eq:NCE_objective_equiv}, \ie, $\fs := \max_{f} \E[(s,a,s',a',y) \sim d^{\tD}]{\log \psi_f(s,a,s',a',y)}$, is characterized by
  \begin{equation*}
    \fs(s,a,s',a') = \frac{\Ppi(s',a' | s,a)}{\Pneg(s',a')}.
  \end{equation*}
\end{lemma}

\begin{proof}
  Note that the objective can be rewritten as
  \begin{align*}
    &\E[(s,a,s',a',y) \sim d^{\tD}]{\log \psi_f(s,a,s',a',y)} \\
    ={}& \int d^{\tD}(s,a,s',a') \prn*{ \sum_{y \in \Y} \Pr(y | s,a,s',a') \log \psi_f(s,a,s',a',y) } \diff s \diff a \diff s' \diff a' \\
    ={}& - \int d^{\tD}(s,a,s',a') \cdot \mathsf{H} \prn[\big]{ \Pr(y | s,a,s',a'); \psi_f(s,a,s',a',y) } \diff s \diff a \diff s' \diff a'.
  \end{align*}
  Here $\mathsf{H}(\cdot; \cdot)$ is the cross entropy between distributions, which, by Gibbs' inequality, is minimized only when
  \begin{equation}\label{eq:NCE_solution-e1}
    \Pr(y | s,a,s',a') = \psi_{\fs}(s,a,s',a',y)
    = \prn*{ \frac{\fs(s,a,s',a')}{1+\fs(s,a,s',a')} }^y \cdot \prn*{ \frac{1}{1+\fs(s,a,s',a')} }^{1-y}.
  \end{equation}
  On the other hand, Bayes' rule states that (note that $\Pr(y | s,a) = \frac{1}{2}$, $\forall y \in \Y$):
  \begin{equation}\label{eq:NCE_solution-e2}
    \Pr(y=1 | s,a,s',a')
    = \frac{\Pr(s',a' | s,a,y=1) \Pr(y=1 | s,a)}{\sum_{y \in \Y} \Pr(s',a' | s,a,y) \Pr(y | s,a)}
    = \frac{\Ppi(s',a' | s,a)}{\Pneg(s',a') + \Ppi(s',a' | s,a)}.
  \end{equation}
  Comparing \eqref{eq:NCE_solution-e1} and \eqref{eq:NCE_solution-e2} gives
  \begin{equation*}
    \frac{\fs(s,a,s',a')}{1+\fs(s,a,s',a')} = \frac{\Ppi(s',a' | s,a)}{\Pneg(s',a') + \Ppi(s',a' | s,a)}
    \quad\implies\quad
    \fs(s,a,s',a') = \frac{\Ppi(s',a' | s,a)}{\Pneg(s',a')}.
  \end{equation*}
  This completes the proof.
\end{proof}

\begin{remark}
  For conciseness, here we slightly abuse the notation $\Pr(\cdot)$ to denote the distribution (density or mass) of joint and conditional distributions involving random variables $(s,a,s',a',y) \sim d^{\tD}$. Specifically, we write $\Pr(\cdots, x, \cdots)$ to indicate an arbitrary value $x$ taken by the random variable, and we also write $\Pr(\cdots, x=x_0, \cdots)$ to emphasize the specific value $x_0$ taken by that random variable.
\end{remark}

\Cref{thm:NCE_solution} is important in that it echoes the form of primal-dual spectral representation in \eqref{eq:spectral_representation_P_pi}. Specifically, we shall take $\Pneg(\cdot,\cdot) \equiv d^{\piref}(\cdot,\cdot)$ for an exact match, which is also implementable using offline data since $d^{\piref}$ can be effectively approximated by sampling the trajectories. We will stick to this choice of $\Pneg$ from now on.

\medskip
\paragraph{Error Bound.} We proceed to show the representation learning error bound for the NCE method, which requires the following regularity assumption on the negative sampling distribution $\Pneg$, or equivalently, as per the choice above, the state-action occupancy measure $d^{\piref}(\cdot,\cdot)$ for the behavior policy $\piref$.

\begin{assumption}[regularity for NCE]\label{assum:NCE_regularity}
  $d^{\piref}_{\P}(s,a) \geq \frac{1}{\Cd} > 0$, $\forall s, a$.
\end{assumption}

We point out that \Cref{assum:NCE_regularity} is a standard assumption for the negative sampling distribution \citep{qiu2022contrastive}, aiming at eliminating the cases where certain transitions are scarcely drawn as negative samples and thus obstruct efficient representation learning for those cases. The assumption is also slightly stronger than the effective behavior policy coverage assumption required by the OLS method (see \Cref{assum:OLS_regularity}).

\begin{theorem}[NCE learning error]\label{thm:NCE_learning_error}
  Under \Cref{assum:initial_distribution,assum:sufficient_sampling,assum:well_defined_family} and the additional \Cref{assum:NCE_regularity} for regularity, let $(\hbphi, \hbmupi)$ be the solution to \eqref{eq:NCE_objective} with $\Pneg(\cdot,\cdot) \equiv d^{\piref}(\cdot,\cdot)$, and set $\hPpi(s',a' | s,a) := d^{\piref}(s', a') \hbphi(s,a)^{\top} \hbmu^{\pi}(s')$. Then, for any $\delta \in (0,1)$, with probability at least $1-\delta$, we have
  \begin{equation*}
    \E[(s,a) \sim d_{\P}^{\piref}]{ \norm[\big]{\Ppi(\cdot,\cdot | s,a) - \hPpi(\cdot,\cdot | s,a)}_1 }
    \leq 2\sqrt{2} (1 + \Cd) \cdot \sqrt{\frac{\log(|\F|/\delta)}{N}}.
  \end{equation*}
\end{theorem}

The proof of \Cref{thm:NCE_learning_error} largely follows the same pathway and techniques established in \citet{qiu2022contrastive}. Nevertheless, our proof is less technically involved since the offline non-episodic setting significantly weakens the correlation between samples. For the sake of completeness, we restate the proof below.

\begin{proof}
  We start by observing $\Pr(y,s',a' | s,a) := \Pr(y | s,a,s',a') \Pr(s',a' | s,a)$, where $\Pr(s',a' | s,a)$ can in turn be calculated using Bayes' rule as follows:
  \begin{align}\label{eq:NCE_sampling_probability}
    \Pr(s',a' | s,a)
    &= \Pr(s',a' | s,a,y=0) \Pr(y=0 | s,a) + \Pr(s',a' | s,a,y=1) \Pr(y=1 | s,a) \nonumber\\
    &= \tfrac{1}{2} \prn[\big]{ \Ppi(s',a' | s,a) + \Pneg(s',a') }.
  \end{align}
  Here we use the fact that the data distribution $d^{\tD}$ implicitly assigns an equal number of labels as $y=0$ and $y=1$ by the design of NCE objective \eqref{eq:NCE_objective}. Since $\Pr(s',a' | s,a)$ is a constant that is independent from $f$, we can further rewrite the NCE objective to be
  \begin{equation}\label{eq:NCE_objective_equiv_MLE}
    \arg\max_{f \in \tF} \brac*{ \widehat{\mathbb{E}}_{(s,a,s',a',y) \in \tD} \brak[\big]{\log \Pr_f(y | s,a,s',a')} }
    = \arg\max_{f \in \tF} \brac*{ \widehat{\mathbb{E}}_{(s,a,s',a',y) \in \tD} \brak[\big]{\log \Pr_f(y,s',a' | s,a)} },
  \end{equation}
  where we define the shorthand notations
  \begin{align*}
    \Pr_f(y | s,a,s',a') &:= \prn*{ \frac{f(s,a,s',a')}{1+f(s,a,s',a')} }^y \cdot \prn*{ \frac{1}{1+f(s,a,s',a')} }^{1-y}, \\
    \Pr_f(y,s',a' | s,a) &:= \prn*{ \frac{f(s,a,s',a') \Pr(s',a' | s,a)}{1+f(s,a,s',a')} }^y \cdot \prn*{ \frac{\Pr(s',a' | s,a)}{1+f(s,a,s',a')} }^{1-y}
  \end{align*}
  for any $f \in \tF$. Note that the right-hand side of \eqref{eq:NCE_objective_equiv_MLE} is in the desired MLE form, with ground-truth conditional density $\Pr_{\fs}(y,s',a' | s,a)$ specified by some $\fs \in \tF$, thanks to the realizability assumption (\Cref{assum:well_defined_family}). Now, using the PAC bound for MLE shown in \citet{agarwal2020flambe} (see \Cref{thm:PAC_MLE}), we have
  \begin{equation*}
    \sum_{i=1}^{N} \E[(s_i,a_i) \sim d_{\P}^{\piref}]{\norm[\big]{\Pr_{\hat{f}}(\cdot,\cdot,\cdot | s_i,a_i) - \Pr_{\fs}(\cdot,\cdot,\cdot | s_i,a_i)}_1^2} \leq 8 \log(|\F| / \delta)
  \end{equation*}
  Since all $(s_i, a_i)$ pairs are sampled i.i.d. from the same distribution $d^{\piref}$, we shall further conclude that
  \begin{equation}\label{eq:NCE_replearn_pac}
    \E[(s,a) \sim d_{\P}^{\piref}]{\norm[\big]{\Pr_{\hat{f}}(\cdot,\cdot,\cdot | s,a) - \Pr_{\fs}(\cdot,\cdot,\cdot | s,a)}_1^2} \leq \frac{8 \log(|\F| / \delta)}{N}.
  \end{equation}
  We proceed to further relate \eqref{eq:NCE_replearn_pac} with the desired format. For this purpose, note that
  \begin{subequations}\label{eq:NCE_learning_error:e1}
  \begin{align}
    &\norm[\big]{\Pr_{\hat{f}}(\cdot,\cdot,\cdot | s,a) - \Pr_{\fs}(\cdot,\cdot,\cdot | s,a)}_1 \nonumber\\
    ={}& \norm[\big]{\Pr_{\hat{f}}(y=1,\cdot,\cdot | s,a) - \Pr_{\fs}(y=1,\cdot,\cdot | s,a)}_1 + \norm[\big]{\Pr_{\hat{f}}(y=0,\cdot,\cdot | s,a) - \Pr_{\fs}(y=0,\cdot,\cdot | s,a)}_1 \label{eq:NCE_learning_error:e1-1}\\
    ={}& 2 \norm*{\frac{\Pr(\cdot,\cdot | s,a)}{1 + \hat{f}(s,a,\cdot,\cdot)} - \frac{\Pr(\cdot,\cdot | s,a)}{1 + \fs(s,a,\cdot,\cdot)}}_1 \label{eq:NCE_learning_error:e1-2}\\
    ={}& 2 \int \frac{\abs[\big]{\hat{f}(s,a,s',a') - \fs(s,a,s',a')} \cdot \Pr(s',a' | s,a)}{ \prn[\big]{ 1 + \hf(s,a,s',a') } \prn[\big]{ 1 + \fs(s,a,s',a') } } \diff s' \diff a', \label{eq:NCE_learning_error:e1-3}
  \end{align}
  \end{subequations}
  where in \eqref{eq:NCE_learning_error:e1-1} we use the definition of $L^1$-norm; in \eqref{eq:NCE_learning_error:e1-2} we use the fact that
  \begin{equation*}
    \Pr_{\hat{f}}(y,\cdot,\cdot | s,a) - \Pr_{\fs}(y,\cdot,\cdot | s,a)
    = (-1)^y \prn*{ \frac{\Pr(\cdot,\cdot | s,a)}{1 + \hat{f}(s,a,\cdot,\cdot)} - \frac{\Pr(\cdot,\cdot | s,a)}{1 + \fs(s,a,\cdot,\cdot)} }.
  \end{equation*}
  Now, plugging \Cref{thm:NCE_solution} and \eqref{eq:NCE_sampling_probability} into the integrand in \eqref{eq:NCE_learning_error:e1-3}, we have
  \begin{subequations}\label{eq:NCE_learning_error:e2}
  \begin{align}
    &\frac{\abs[\big]{\hat{f}(s,a,s',a') - \fs(s,a,s',a')} \cdot \Pr(s',a' | s,a)}{ \prn[\big]{1 + \hf(s,a,s',a')} \prn[\big]{ 1 + \fs(s,a,s',a') } } \nonumber\\
    ={}& \frac{ \abs[\big]{\Ppi(s',a' | s,a) / \Pneg(s',a') - \hf(s,a,s',a')} \cdot \frac{1}{2} \prn[\big]{ \Ppi(s',a' | s,a) + \Pneg(s',a') }}{ \prn[\big]{1 + \hf(s,a,s',a')} \prn[\big]{1 + \Ppi(s',a' | s,a) / \Pneg(s',a')} } \label{eq:NCE_learning_error:e2-1}\\
    ={}& \frac{\abs[\big]{\Ppi(s',a' | s,a) - \Pneg(s',a') \hf(s,a,s',a')}}{ 2\prn[\big]{1 + \hf(s,a,s',a')} } \label{eq:NCE_learning_error:e2-2}\\
    \geq{}& \frac{\abs[\big]{\Ppi(s',a' | s,a) - \Pneg(s',a') \hf(s,a,s',a')}}{ 2(1 + \Cd) }, \label{eq:NCE_learning_error:e2-3}
  \end{align}
  \end{subequations}
  where we use the upper bound $\hf(s,a,s',a') = \hPpi(s',a' | s,a) / d^{\piref}(s',a') \leq \Cd$ in \eqref{eq:NCE_learning_error:e2-3}. Consequently,
  \begin{subequations}\label{eq:NCE_learning_error:e3}
  \begin{align}
    &\norm[\big]{\Ppi(\cdot,\cdot | s,a) - \hPpi(\cdot,\cdot | s,a)}_1 \nonumber\\
    ={}& \int \abs[\big]{\Ppi(s',a' | s,a) - \Pneg(s',a') \hf(s,a,s',a')} \diff s' \diff a' \label{eq:NCE_learning_error:e3-1}\\
    \leq{}& 2(1+\Cd) \int \frac{\abs[\big]{\hat{f}(s,a,s',a') - \fs(s,a,s',a')} \cdot \Pr(s',a' | s,a)}{ \prn[\big]{ 1 + \hf(s,a,s',a') } \prn[\big]{ 1 + \fs(s,a,s',a') } } \diff s' \diff a' \label{eq:NCE_learning_error:e3-2}\\
    ={}& (1+\Cd) \norm[\big]{\Pr_{\hat{f}}(\cdot,\cdot,\cdot | s,a) - \Pr_{\fs}(\cdot,\cdot,\cdot | s,a)}_1, \label{eq:NCE_learning_error:e3-3}
  \end{align}
  \end{subequations}
  where we use $\hPpi(\cdot,\cdot | s,a) = \Pneg(\cdot,\cdot) \hf(s,a,\cdot,\cdot)$ in \eqref{eq:NCE_learning_error:e3-1}, \eqref{eq:NCE_learning_error:e2} in \eqref{eq:NCE_learning_error:e3-2}, and \eqref{eq:NCE_learning_error:e1} in \eqref{eq:NCE_learning_error:e3-3}.
  Finally,
  \begin{subequations}\label{eq:NCE_learning_error:e4}
  \begin{align}
    &\E[(s,a) \sim d_{\P}^{\piref}]{ \norm[\big]{\Ppi(\cdot,\cdot | s,a) - \hPpi(\cdot,\cdot | s,a)}_1 } \nonumber\\
    \leq{}& \sqrt{ \E[(s,a) \sim d_{\P}^{\piref}]{ \norm[\big]{\Ppi(\cdot,\cdot | s,a) - \hPpi(\cdot,\cdot | s,a)}_1^2 } } \label{eq:NCE_learning_error:e4-1}\\
    \leq{}& \sqrt{ (1+\Cd)^2 \E[(s,a) \sim d_{\P}^{\piref}]{\norm[\big]{\Pr_{\hat{f}}(\cdot,\cdot,\cdot | s,a) - \Pr_{\fs}(\cdot,\cdot,\cdot | s,a)}_1^2} } \label{eq:NCE_learning_error:e4-2}\\
    \leq{}& 2\sqrt{2} (1 + \Cd) \cdot \sqrt{\frac{\log(|\F|/\delta)}{N}}, \label{eq:NCE_learning_error:e4-3}
  \end{align}
  \end{subequations}
  where we use Cauchy-Schwartz inequality in \eqref{eq:NCE_learning_error:e4-1}, \eqref{eq:NCE_learning_error:e3} in \eqref{eq:NCE_learning_error:e4-2}, and \eqref{eq:NCE_replearn_pac} in \eqref{eq:NCE_learning_error:e4-3}.
\end{proof}
  \section{Sample Complexity Guarantee}\label{sec:apdx-analysis}

In this appendix, we derive the sample complexity guarantee for the proposed \algname algorithm, assuming a known bound on the representation learning error induced by the \replearn subroutine (see \Cref{claim:representation_learning_error}). As discussed in the main text, the objective is to bound the estimation error $\Err := \hat{\rho}(\pi) - \rho(\pi)$, which can be intuitively split into the following three terms that are easier to bound:
\begin{equation*}
  \Err \quad=\quad 
  \underbrace{\hat{\rho}(\pi) - \bar{\rho}(\pi)}_{\textrm{statistical error}}
  \;+\;\; \underbrace{\bar{\rho}(\pi) - \rho_{\hP}(\pi)}_{\textrm{dataset error}}
  \;\;+\; \underbrace{\rho_{\hP}(\pi) - \rhoP(\pi)}_{\textrm{representation error}}
  .
\end{equation*}
We point out that the statistical error results from replacing the expectation with empirical estimates, the dataset error comes from the offline dataset that samples transitions from the true transition kernel $\Ppi$ instead of the learned kernel $\hPpi$, and the representation error accounts for the error induced by plugging in the learned representation $(\hbphi, \hbmupi)$ instead of the ground truth $(\sbphi, \sbmupi)$ into the DICE estimator.

As described in the proof sketch, for the rest of this appendix, we provide an upper bound for each of these three terms, and eventually conclude with an overall sample complexity guarantee.

\paragraph{Representation Error.} We start by bounding the representation error term, which by intuition should be a direct consequence of the representation learning error shown in \Cref{claim:representation_learning_error}.

\begin{lemma}\label{thm:complexity_representation_error}
  Conditioned on the event that the inequality in \Cref{claim:representation_learning_error} holds, under \Cref{assum:initial_distribution,assum:sufficient_sampling,assum:well_defined_family}, 
  \begin{equation*}
    \rho_{\hP}(\pi) - \rhoP(\pi)
    \leq \frac{\gamma \Cpi}{(1-\gamma)^2} \cdot \xi(|\F|, N, \delta).
  \end{equation*}
\end{lemma}

\begin{proof}
  By the well-known Simulation Lemma (see \Cref{thm:simulation_lemma}), we have
  \begin{subequations}\label{eq:complexity_representation_error}
  \begin{align}
    &\rho_{\hP}(\pi) - \rhoP(\pi) \nonumber\\
    ={}& \frac{\gamma}{1-\gamma} \E[(s,a) \sim d_{\P}^{\pi}]{ \E[s' \sim \hP(\cdot | s,a)]{V_{\hP}^{\pi}(s')} - \E[s' \sim \P(\cdot | s,a)]{V_{\hP}^{\pi}(s')}} \label{eq:complexity_representation_error:1}\\
    ={}& \frac{\gamma}{1-\gamma} \E[(s,a) \sim d_{\P}^{\pi}]{ \E[(s',a') \sim \hPpi(\cdot,\cdot | s,a)]{Q_{\hP}^{\pi}(s',a')} - \E[(s',a') \sim \Ppi(\cdot,\cdot | s,a)]{Q_{\hP}^{\pi}(s',a')} } \label{eq:complexity_representation_error:2}\\
    ={}& \frac{\gamma}{1-\gamma} \int d_{\P}^{\pi}(s,a) \diff s \diff a \int Q_{\hP}^{\pi}(s',a') \prn*{ \hPpi(s',a' | s,a) - \Ppi(s',a' | s,a) } \diff s' \diff a' \label{eq:complexity_representation_error:3}\\
    \leq{}& \frac{\gamma}{(1-\gamma)^2} \int \Cpi d_{\P}^{\piref}(s,a) \diff s \diff a \int \abs[\big]{ \Ppi(s',a' | s,a) - \hPpi(s',a' | s,a) } \diff s' \diff a' \label{eq:complexity_representation_error:4}\\
    ={}& \frac{\gamma \Cpi}{(1-\gamma)^2} \E[(s,a) \sim d_{\P}^{\piref}]{ \norm[\big]{ \Ppi(s',a' | s,a) - \hPpi(s',a' | s,a) }_1 } \label{eq:complexity_representation_error:5}\\
    \leq{}& \frac{\gamma \Cpi}{(1-\gamma)^2} \cdot \xi(|\F|, N, \delta), \label{eq:complexity_representation_error:6}
  \end{align}
  \end{subequations}
  where in \eqref{eq:complexity_representation_error:1} we use the Simulation Lemma; in \eqref{eq:complexity_representation_error:2} we use the relationship between value functions; in \eqref{eq:complexity_representation_error:4} we plug in $d_{\P}^{\pi}(s,a) \leq \Cpi d_{\P}^{\piref}(s,a)$ (\Cref{assum:sufficient_sampling}) and the fact that $Q_{\hP}^{\pi}(\cdot, \cdot) \leq \frac{1}{1-\gamma}$; in \eqref{eq:complexity_representation_error:6} we use \Cref{claim:representation_learning_error}.
\end{proof}

\paragraph{Dataset Error.} The dataset error can be accounted for by a bounded difference in the objective function, which turns out to be another consequence of the representation learning error. For this purpose, we first show the following technical lemma that formalizes the above intuition.

\begin{lemma}\label{thm:lemma_min_max_diff}
  \(
    \min\limits_{\bx \in \X} \max\limits_{\by \in \Y} F_1(\bx, \by) - \min\limits_{\bx \in \X} \max\limits_{\by \in \Y} F_2(\bx, \by)
    \leq \max\limits_{\bx \in \X, \by \in \Y} \abs[\big]{ F_1(\bx, \by) - F_2(\bx, \by) }.
  \)
\end{lemma}

\begin{proof}
  Let $\varepsilon := \max\limits_{\bx, \by} \abs[\big]{ F_1(\bx, \by) - F_2(\bx, \by) }$. Then we have
  \begin{subequations}\label{eq:lemma_min_max_diff}
  \begin{align}
    \min_{\bx} \max_{\by} F_1(\bx, \by)
    &\leq \min_{\bx} \brac*{ \max_{\by} F_2(\bx, \by) + \max_{\by} \brac*{ F_1(\bx, \by) - F_2(\bx, \by) } } \label{eq:lemma_min_max_diff:1}\\
    &\leq \min_{\bx} \brac*{ \max_{\by} F_2(\bx, \by) + \varepsilon } \label{eq:lemma_min_max_diff:2}\\
    &= \min_{\bx} \max_{\by} F_2(\bx, \by) + \varepsilon, \label{eq:lemma_min_max_diff:3}
  \end{align}
  \end{subequations}
  where in \eqref{eq:lemma_min_max_diff:1} we use the fact that $\max_{\by}\brac{ f(\by) + g(\by) } \leq \max_{\by} f(\by) + \max_{\by} g(\by)$.
\end{proof}

Now we are ready to show the following lemma regarding dataset error.

\begin{lemma}\label{thm:complexity_dataset_error}
  Conditioned on the event that the inequality in \Cref{claim:representation_learning_error} holds, under \Cref{assum:initial_distribution,assum:sufficient_sampling,assum:well_defined_family},
  \begin{equation*}
    \bar{\rho}(\pi) - \rho_{\hP}(\pi) \leq \frac{\Cpi}{1 - \gamma} \cdot \xi(|\F|, N, \delta).
  \end{equation*}
\end{lemma}

\begin{proof}
  For the sake of clarity, denote the optimization objectives of $\bar{\rho}(\pi)$ and $\rho_{\hP}(\pi)$ as follows:
  \begin{equation*}
    \bar{\rho}(\pi) = \min_{\btheta_Q} \max_{\bomega_d} \bar{F}(\btheta_Q, \bomega_d),\qquad
    \rho_{\hP}(\pi) = \min_{\btheta_Q} \max_{\bomega_d} \hat{F}(\btheta_Q, \bomega_d).
  \end{equation*}
  Then we can show that
  \begin{subequations}\label{eq:complexity_dataset_error}
  \begin{align}
    \abs{\bar{\rho}(\pi) - \rho_{\hP}(\pi)}
    &= \abs*{ \min_{\btheta_Q} \max_{\bomega_d} \bar{F}(\btheta_Q, \bomega_d) - \min_{\btheta_Q} \max_{\bomega_d} \hat{F}(\btheta_Q, \bomega_d) }
    \leq \abs*{ \bar{F}(\btheta_Q, \bomega_d) - \hat{F}(\btheta_Q, \bomega_d) } \label{eq:complexity_dataset_error:1}\\
    &= \Biggl\vert \int d_{\P}^{\piref}(s,a) \prn*{ \Ppi(s',a' | s,a) - \hPpi(s',a' | s,a) } \prn[\big]{\hbmu^{\pi}(s,a)^{\top} \bomega_d} \cdot{} \nonumber\\
      &\hspace{9em} \prn*{r(s,a) + \gamma \hbphi(s',a')^{\top} \btheta_Q - \hbphi(s,a)^{\top} \btheta_Q} \diff s \diff a \diff s' \diff a' \Biggr\vert \label{eq:complexity_dataset_error:2}\\
    &\leq \int d_{\P}^{\piref}(s,a) \abs[\big]{ \Ppi(s',a' | s,a) - \hPpi(s',a' | s,a) } \cdot \abs[\big]{\hbmu^{\pi}(s,a)^{\top} \bomega_d} \cdot \nonumber\\
      &\hspace{9em} \abs[\big]{r(s,a) + \gamma \hbphi(s',a')^{\top} \btheta_Q - \hbphi(s,a)^{\top} \btheta_Q} \diff s \diff a \diff s' \diff a' \label{eq:complexity_dataset_error:3}\\
    &\leq \E[(s,a) \sim d_{\P}^{\piref}]{\norm[\big]{ \Ppi(\cdot,\cdot | s,a) - \hPpi(\cdot,\cdot | s,a) }_1 \cdot \Cpi \cdot \tfrac{1}{1-\gamma}} \label{eq:complexity_dataset_error:4}\\
    &= \frac{\Cpi}{1-\gamma} \cdot \xi(|\F|, N, \delta), \label{eq:complexity_dataset_error:5}
  \end{align}
  \end{subequations}
  where in \eqref{eq:complexity_dataset_error:1} we use \Cref{thm:lemma_min_max_diff}; in \eqref{eq:complexity_dataset_error:3} we use the integral triangle inequality; in \eqref{eq:complexity_dataset_error:4} we plug in $\abs[\big]{\hbmu^{\pi}(s,a)^{\top} \bomega_d} \leq \Cpi$ and $\abs[\big]{ r(s,a) + \gamma \hbphi(s',a')^{\top} \btheta_Q - \hbphi(s,a)^{\top} \btheta_Q } \leq \frac{1}{1-\gamma}$ (see \Cref{remark:numerical}); in \eqref{eq:complexity_dataset_error:5} we use \Cref{claim:representation_learning_error}.
\end{proof}

\paragraph{Statistical Error.} Finally, the statistical error is caused by replacing the expectations with their empirical estimations, which can be bounded by Hoeffding's concentration inequality (see \Cref{thm:hoeffding_concentration}).

\begin{lemma}\label{thm:complexity_statistical_error}
  Under \Cref{assum:initial_distribution,assum:sufficient_sampling,assum:well_defined_family}, with probability at least $1-\delta$, we have
  \begin{equation*}
    \hat{\rho}(\pi) - \bar{\rho}(\pi) \leq \frac{\Cpi}{1-\gamma} \sqrt{ \frac{\log (1/2\delta)}{2N} }.
  \end{equation*}
\end{lemma}

\begin{proof}
  For clarity, label the samples in $\D$ as $\D = \set{(s_i, a_i, s'_i, a'_i) \mid i \in [N]}$, and define
  \begin{equation*}
    F(s,a,s',a') := \prn[\big]{\hbmu^{\pi}(s,a)^{\top} \bomega_d} \prn[\big]{r(s,a) + \gamma \hbphi(s',a')^{\top} \btheta_Q - \hbphi(s,a)^{\top} \btheta_Q}.
  \end{equation*}
  Note that $\abs[\big]{\hbmu^{\pi}(s,a)^{\top} \bomega_d} \leq \Cpi$ and $\abs[\big]{ r(s,a) + \gamma \hbphi(s',a')^{\top} \btheta_Q - \hbphi(s,a)^{\top} \btheta_Q } \leq \frac{1}{1-\gamma}$ (see \Cref{remark:numerical}), we have
  \begin{equation*}
    \abs[\big]{F(s,a,s',a')} \leq \frac{\Cpi}{1-\gamma},~ \forall s,a,s',a'.
  \end{equation*}
  Therefore, by Hoeffding's inequality (see \Cref{thm:hoeffding_concentration}), we conclude that
  \begin{equation*}
    \Prob{ \abs*{ \frac{1}{N} \sum_{i=1}^{N} F(s_i,a_i,s'_i,a'_i) - \E[\begin{subarray}{l} s \sim d^{\piref}(\cdot),~ a \sim \piref(a | s),\\ s' \sim \P(\cdot | s,a),~ a' \sim \pi(\cdot | s') \end{subarray}]{ F(s,a,s',a') } } > t } \leq 2 \exp\prn*{- \frac{2 N t^2}{4 (\Cpi)^2 / (1-\gamma)^2}}.
  \end{equation*}
  Or equivalently, with probability at least $1-\delta$, we have
  \begin{equation*}
    \abs*{ \hE[\begin{subarray}{c} (s,a,s') \sim \D,\\ a' \sim \pi(\cdot | s') \end{subarray}]{ F(s,a,s',a') } - \E[\begin{subarray}{l} s \sim d^{\piref}(\cdot),~ a \sim \piref(a | s),\\ s' \sim \P(\cdot | s,a),~ a' \sim \pi(\cdot | s') \end{subarray}]{ F(s,a,s',a') } }
    \leq \frac{\Cpi}{1-\gamma} \sqrt{ \frac{\log (1/2\delta)}{2N} }.
  \end{equation*}
  Finally, the conclusion follows from \Cref{thm:lemma_min_max_diff} using the same argument as above.
\end{proof}

\paragraph{Conclusion.} Now we are ready to prove the Main Theorem.

\begin{restate}[\Cref{thm:main_theorem}]
  Suppose \Cref{claim:representation_learning_error} holds for the $\replearn(\F, \D, \pi)$ subroutine. Then under \Cref{assum:initial_distribution,assum:sufficient_sampling,assum:well_defined_family}, with probability at least $1-\delta$, we have
  \begin{equation*}
    \Err \leq \frac{\Cpi}{1-\gamma} \sqrt{\frac{\log(1 / \delta)}{2N}} + \frac{\Cpi}{(1-\gamma)^2}  \cdot \xi(|\F|, N, \delta/2).
  \end{equation*}
\end{restate}

\begin{proof}
  Consider the following high-probability events:
  \begin{align*}
    \mathcal{C}_1:~& \E[(s,a) \sim d^{\piref}_{\P}]{\norm[\big]{\hP^{\pi}(\cdot, \cdot | s,a) - \Ppi(\cdot,\cdot | s,a)}_1} \leq \xi(|\F|, N, \delta/2), \\
    \mathcal{C}_2:~& \abs*{ \hE[\begin{subarray}{c} (s,a,s') \sim \D,\\ a' \sim \pi(\cdot | s') \end{subarray}]{ F(s,a,s',a') } - \E[\begin{subarray}{l} s \sim d^{\piref}(\cdot),~ a \sim \piref(a | s),\\ s' \sim \P(\cdot | s,a),~ a' \sim \pi(\cdot | s') \end{subarray}]{ F(s,a,s',a') } }
    \leq \frac{\Cpi}{1-\gamma} \sqrt{ \frac{\log (1/\delta)}{2N} }.
  \end{align*}
  As per \Cref{claim:representation_learning_error} and \Cref{thm:complexity_statistical_error}, we know $\Prob{\mathcal{C}_i} \geq 1 - \delta/2$ ($i = 1,2$). Hence by Union Bound we have
  \begin{equation*}
    \Prob{\mathcal{C}_1 \cap \mathcal{C}_2} \geq 1 - \delta.
  \end{equation*}
  On the other hand, conditioned on $\mathcal{C}_1 \cap \mathcal{C}_2$, \Cref{thm:complexity_representation_error}, \Cref{thm:complexity_dataset_error} and \Cref{thm:complexity_statistical_error} in combination guarantee
  \begin{align*}
    \Err &= \hat{\rho}(\pi) - \bar{\rho}(\pi) + \bar{\rho}(\pi) - \rho_{\hP}(\pi) + \rho_{\hP}(\pi) - \rhoP(\pi) \\
    &\leq \frac{\Cpi}{1-\gamma} \sqrt{ \frac{\log (1/\delta)}{2N} } + \frac{\Cpi}{1 - \gamma} \cdot \xi(|\F|, N, \delta/2) + \frac{\gamma \Cpi}{(1-\gamma)^2}  \xi(|\F|, N, \delta/2) \\
    &= \frac{\Cpi}{1-\gamma} \sqrt{ \frac{\log (1/\delta)}{2N} } + \frac{\Cpi}{(1-\gamma)^2} \cdot \xi(|\F|, N, \delta/2)
  \end{align*}
  This completes the proof.
\end{proof}

For completeness, we also include the corollaries of the Main Theorem that characterize the overall sample complexity of our \algname algorithm using OLS and NCE representation learning methods.

\begin{corollary}[sample complexity of OLS-based \algname]
   Under \Cref{assum:initial_distribution,assum:sufficient_sampling,assum:well_defined_family} and the additional \Cref{assum:OLS_regularity} for regularity, let $(\hbphi, \hbmupi)$ be the solution to the OLS problem \eqref{eq:OLS_objective}. Then, for any $\delta \in (0,1)$, with probability at least $1-\delta$, we have
  \begin{equation*}
    \Err \leq \frac{\Cpi}{1-\gamma} \sqrt{ \frac{\log (1/\delta)}{2N} } + \frac{\Cpi \sqrt{\CP \Creg}}{(1-\gamma)^2} \cdot \sqrt{\frac{\log(2|\F|/\delta)}{N}}
    \lesssim \frac{1}{(1-\gamma)^2} \sqrt{\frac{\log(|\F|/\delta)}{N}},
  \end{equation*}
  where $\Creg = \frac{4}{3}\sqrt{\Ccov} + 8 \Ccov$ is a universal constant determined by the PAC bound for OLS..
\end{corollary}

\begin{corollary}[sample complexity of NCE-based \algname]
   Under \Cref{assum:initial_distribution,assum:sufficient_sampling,assum:well_defined_family} and the additional \Cref{assum:NCE_regularity} for regularity, let $(\hbphi, \hbmupi)$ be the solution to the NCE problem \eqref{eq:NCE_objective} with $\Pneg(\cdot,\cdot) \equiv d^{\piref}(\cdot,\cdot)$. Then, for any $\delta \in (0,1)$, with probability at least $1-\delta$, we have
  \begin{equation*}
    \Err \leq \frac{\Cpi}{1-\gamma} \sqrt{ \frac{\log (1/\delta)}{2N} } + \frac{2\sqrt{2} \Cpi (1+\Cd)}{(1-\gamma)^2} \cdot \sqrt{\frac{\log(2|\F|/\delta)}{N}}
    \lesssim \frac{1}{(1-\gamma)^2} \sqrt{\frac{\log(|\F|/\delta)}{N}}.
  \end{equation*}
\end{corollary}

\begin{remark}[Sampling the dataset]
  Throughout this paper, we have been slightly abusing the notation $(s,a,s') \sim \D$, which is a little subtle in practice since only trajectories (rather than transitions) are collected. To ensure the correct data distribution $d^{\D}(s,a) = d^{\piref}(s,a)$, we shall first randomly sample the trajectories, within which we sample each transition $(s_t, a_t, s_{t+1}, a_{t+1})$ with probability $(1-\gamma) \gamma^t$.
\end{remark}
  \section{Technical Lemmas}\label{sec:apdx-3_technical}

In this final appendix, we include all the technical lemmas used in the previous sections.

\subsection{\titlemath{f}-Divergence}\label{sec:apdx-4_Fenchel_conjugate}

\begin{definition}[$f$-divergence]
  Let $\P$ and $\Q$ be two probabilities distribution over a sample space $\X$, such that $\P$ is absolutely continuous with respect to $\Q$. Given a convex function $f: \R_{\geq 0} \to \R$ such that $f(1) = 0$ and $f(0) := \lim_{t \to 0^+} f(t)$. Then the \textit{$f$-divergence} of $\P$ with respect to $\Q$ is defined as
  \begin{equation*}
    \Div_f(\P \Vert \Q) := \int_{\X} f\prn*{\frac{\diff \P}{\diff \Q}} \diff \Q.
  \end{equation*}
\end{definition}

The following \emph{variational representation} of $f$-divergences is well-known in literature.

\begin{lemma}[variational representation using Fenchel conjugate]\label{thm:f_divergence_duality}
  Let $\F$ denote the class of measurable real valued functions on $\X$ that is absolutely integratable with respect to $\Q$. Then
  \begin{equation*}
    \Div_f(\P \Vert \Q) = \sup_{g \in \F} \brac[\Big]{ \E[x \sim \P]{g(x)} - \E[x \sim \Q]{f_*(g(x))} },
  \end{equation*}
  where $f_*$ is the Fenchel conjugate of $f$. Further, if $f$ is differentiable, then the optimal dual variable is given by
  \begin{equation*}
    g^{\star}(x) = f'\prn*{\tfrac{\diff \P}{\diff \Q}}
    \quad\implies\quad
    \Div_f(\P \Vert \Q) = \E[x \sim \P]{f'\prn*{\tfrac{\diff \P}{\diff \Q}}} - \E[x \sim \Q]{f_*\prn*{f'\prn*{\tfrac{\diff \P}{\diff \Q}}}}
  \end{equation*}
\end{lemma}

\begin{proof}
  See Theorem 4.4 in \citet{broniatowski2006minimization}.
\end{proof}

\subsection{Concentration Inequalities}

\begin{lemma}[Hoeffding's inequality, \citet{hoeffding1994probability}]\label{thm:hoeffding_concentration}
  Let $X_1, X_2, \cdots, X_N$ be i.i.d. random variables with mean $\mu$ and taking values in $[a,b]$ almost surely. Then for any $\varepsilon > 0$ we have
  \begin{equation*}
    \Prob{\abs*{\frac{1}{N} \sum_{i=1}^{N} X_i - \mu} > \varepsilon} \leq 2 \exp\prn*{- \frac{2 N \varepsilon^2}{(b-a)^2}}.
  \end{equation*}
  In other words, with probability at least $1-\delta$, we have
  \begin{equation*}
    \abs*{\frac{1}{N} \sum_{i=1}^{N} X_i - \mu} \leq (b-a) \sqrt{\frac{\log(1/2\delta)}{2N}}.
  \end{equation*}
\end{lemma}

\begin{lemma}[Bernstein's inequality, \citet{bernstein1924modification}]\label{thm:bernstein_concentration}
  Let $X_1, X_2, \cdots, X_N$ be i.i.d. random variables with mean $\mu$, variance $\sigma^2$, and bounded range $|X_i - \mu| \leq B$ almost surely. Then with probability at least $1-\delta$, we have
  \begin{equation*}
    \pm\prn*{ \frac{1}{N} \sum_{i=1}^{N} X_i - \mu} \leq \sigma \sqrt{\frac{2 \log(1/\delta)}{N}} + \frac{B \log(1/\delta)}{3N}.
  \end{equation*}
\end{lemma}

\subsection{Statistical Learning: PAC Bounds}

In this section, we present the standard PAC bounds for OLS and MLE. Although these are both classic results, we fail to trace back to the original literature of the former, and thus provide a short proof here for completeness.

\begin{lemma}[PAC bound for OLS, fast rate]\label{thm:PAC_regression_fast_rate}
  Consider a regression problem over a finite family $\F = \set{f: \X \to [a,b]}$ of bounded functions with data distribution $(X,Y) \sim \M$, where the objective is to solve for
  \begin{equation*}
    \arg\min_{f \in \F} \L(f),\quad \textrm{where}~ \L(f) := \E[(X,Y) \sim \M]{(f(X) - Y)^2}.
  \end{equation*}
  Suppose the regression function $\fs(x) := \E{Y \mid X=x} \in \F$ (realizability), and we have access to i.i.d. sample $(x_i, y_i) \sim \M$, $\forall i \in [N]$. Let the Empirical Risk Minimization (ERM) estimator be
  \begin{equation*}
    \hf := \arg\min_{f \in \F} \hat{\L}(f),\quad \textrm{where}~ \hat{\L}(f) := \frac{1}{N} \sum_{i=1}^{N} (f(x_i) - y_i)^2.
  \end{equation*}
  Then, with probability at least $1-\delta$, the ERM estimator induces a regret that is at most
  \begin{equation*}
    \L(\hf) \leq \L(\fs) + \Creg \frac{\log(|\F|/\delta)}{N}.
  \end{equation*}
  Suppose further that the ground truth is deterministic such that $y = \fs(x)$ for some $\fs \in \F$, in which case we have
  \begin{equation*}
    \L(\hf) \leq \Creg \frac{\log(|\F|/\delta)}{N}.
  \end{equation*}
  Here $\Creg = 8(b-a)^2 + \frac{4}{3}(b-a)$ is a universal constant depending only on the range $[a,b]$.
\end{lemma}

\begin{proof}
  Define a random variable $Z_i := (f(X_i) - Y_i)^2 - (\fs(X_i) - Y_i)^2$, such that
  \begin{subequations}\label{eq:PAC_regression:e1}
  \begin{align}
    \E[(X_i, Y_i) \sim \M]{Z_i(f)} 
    ={}& \E[(X_i, Y_i) \sim \M]{(f(X_i)-Y_i)^2 - (\fs(X_i)-Y_i)^2} \\
    ={}& \E[(X_i, Y_i) \sim \M]{\prn[\big]{ (f(X_i)-\fs(X_i)) + (\fs(X_i)-Y_i) }^2 - (\fs(X_i)-Y_i)^2} \label{eq:PAC_regression:e1-2}\\
    ={}& \E[(X_i, Y_i) \sim \M]{(f(X_i)-\fs(X_i))^2} + 2 \E[(X_i, Y_i) \sim \M]{(f(X_i)-\fs(X_i))(\fs(X_i)-Y_i)} \label{eq:PAC_regression:e1-3}\\
    ={}& \E[(X_i, Y_i) \sim \M]{(f(X_i)-\fs(X_i))^2} =: \Err(f),
  \end{align}
  \end{subequations}
  where in \eqref{eq:PAC_regression:e1-3} we use the following fact $\E[(X_i, Y_i) \sim \M]{(f(X_i)-\fs(X_i))(\fs(X_i)-Y_i)} = \Es[X_i]{(f(X_i)-\fs(X_i)) \cdot \E[Y_i \sim \M(\cdot \mid X_i)]{\fs(X_i)-Y_i}} = 0$ that directly follows from the definition of $\fs$. Similarly, for any $t \in [T]$,
  \begin{subequations}\label{eq:PAC_regression:e3}
  \begin{align}
    \Var[(X_i, Y_i) \sim \M]{Z_i(f)} 
    ={}& \E[(X_i, Y_i) \sim \M]{Z_i(f)^2} - \prn*{\E[(X_i, Y_i) \sim \M]{Z_i(f)}}^2 \label{eq:PAC_regression:e3-1}\\
    \leq{}& \E[(X_i, Y_i) \sim \M]{\prn*{(f(X_i)-Y_i)^2 - (\fs(X_i)-Y_i)^2}^2} \label{eq:PAC_regression:e3-2}\\
    ={}& \E[(X_i, Y_i) \sim \M]{(f(X_i)-\fs(X_i))^2 (f(X_i)+\fs(X_i)-2Y_i)^2} \label{eq:PAC_regression:e3-3}\\
    \leq{}& 4(b-a)^2 \E[(X_i, Y_i) \sim \M]{(f(X_i)-\fs(X_i))^2}
    = 4(b-a)^2 \Err(f).
  \end{align}
  \end{subequations}
  where in \eqref{eq:PAC_regression:e3-1} we simply drop the second term, and in \eqref{eq:PAC_regression:e3-3} we use the fact $f(X_i)+\fs(X_i)-2Y_i \in [-2(b-a),2(b-a)]$ as $f(X_i), \fs(X_i), Y_i \in [a,b]$. Further, for any $x \in \mathcal{X}$, $y \in \mathcal{Y}$ and $f \in \mathcal{F}$, we have $f(x) - y \in [-(b-a),b-a]$, implying $Z_i(f) \in [-(b-a),b-a]$ and $\E{Z_i(f)} \in [-(b-a),(b-a)]$. Therefore, $\abs{Z_i(f) - \E{Z_i(f)}} \leq 2(b-a)$. Then by Bernstein's inequality (\Cref{thm:bernstein_concentration}), we conclude that, with probability at least $1-\delta$,
  \begin{equation}\label{eq:PAC_regression:e4}
    \E{Z_i(f)} - \frac{1}{N} \sum_{i=1}^{N} Z_i(f) \leq \sqrt{\Var{Z_i(f)}} \sqrt{\frac{2 \log(1/\delta)}{N}} + \frac{2(b-a) \log(1/\delta)}{3N}.
  \end{equation}
  To proceed, plug \eqref{eq:PAC_regression:e1} and \eqref{eq:PAC_regression:e3} into \eqref{eq:PAC_regression:e4}, and we have
  \begin{subequations}\label{eq:PAC_regression:e5}
  \begin{align}
    \Err(f) - (\hat{\L}(f) - \hat{\L}(\fs))
    &\leq 2(b-a) \sqrt{\Err(f)} \sqrt{\frac{2 \log(1/\delta)}{N}} + \frac{2(b-a) \log(1/\delta)}{3N} \label{eq:PAC_regression:e5-1}\\
    &\leq \prn*{\frac{1}{2} \Err(f) + \frac{4(b-a)^2 \log(1/\delta)}{N}} + \frac{2(b-a) \log(1/\delta)}{3N},
  \end{align}
  \end{subequations}
  where in \eqref{eq:PAC_regression:e5-1} we apply the AM-GM inequality. Finally, we rearrange the terms to obtain
  \begin{equation}\label{eq:1-3-e3}
    \Err(f) \leq 2(\hat{\L}(f) - \hat{\L}(\fs)) + \frac{\Creg \log(1/\delta)}{N}
  \end{equation}
  for any fixed $f \in \mathcal{F}$, with probability at least $1-\delta$. Finally, we take the union bound with respect to all $f \in \mathcal{F}$, such that with probability at least $1-\delta$, we have
  \begin{equation}\label{eq:PAC_regression:e6}
    \Err(f) \leq 2(\hat{\L}(f) - \hat{\L}(\fs)) + \frac{\Creg \log(|\mathcal{F}|/\delta)}{N},~ \forall f \in \mathcal{F}.
  \end{equation}
  In particular, \eqref{eq:PAC_regression:e6} also applies to the ERM estimator $\hf$, which gives
  \begin{equation}\label{eq:PAC_regression:e7}
    \Err(\hf)
    \leq 2(\hat{\L}(\hf) - \hat{\L}(\fs)) + \frac{\Creg \log(|\mathcal{F}|/\delta)}{N}
    \leq \frac{\Creg \log(|\mathcal{F}|/\delta)}{N}.
  \end{equation}
  Here we use the inequality $\hat{\L}(\hf) \leq \hat{\L}(\fs)$, as $\hf$ minimizes $\hat{\L}(\cdot)$ within $\mathcal{F}$. This completes the proof.
\end{proof}

\begin{lemma}[PAC bound for MLE, \citet{agarwal2020flambe}]\label{thm:PAC_MLE}
  Consider a conditional probability estimation problem over a finite family $\F = \set{f: (\X \times \Y) \to \R}$, where the objective is to estimate $\fs(x, y) := \P(y | x)$. Suppose the ground truth $\fs \in \F$ (realizability), and we have access to (potentially correlated) samples $\set{(x_i, y_i) \mid i \in [N]}$ such that $x_i \sim \D_i$ ($\D_i$ is allowed to depend on $(x_{1:i-1}, y_{1:i-1})$, forming a martigale process) and $y_i \sim \P(\cdot | x_i)$. Let the Maximum Likelihood Estimator (MLE) be
  \begin{equation*}
    \hf := \arg\max_{f \in \F} \sum_{i=1}^{N} \log f(x_i, y_i).
  \end{equation*}
  Then, with probability at least $1-\delta$, the error of the MLE estimator is bounded as follows:
  \begin{equation*}
    \sum_{i=1}^{N} \E[x \sim \D_i]{\norm[\big]{\hf(x, \cdot) - \fs(x, \cdot)}_1^2} \leq 8 \log(|\F|/\delta).
  \end{equation*}
  Specifically, when $\set{(x_i, y_i) \mid i \in [N]}$ are i.i.d. samples from a dataset $\D$, we have
  \begin{equation*}
    \E[x \sim \D]{\norm[\big]{\hf(x, \cdot) - \fs(x, \cdot)}_1^2} \leq \frac{8 \log(|\F|/\delta)}{N}.
  \end{equation*}
\end{lemma}

\subsection{Simulation Lemma in MDPs}

The following Simulation Lemma is a simplified version of Lemma 21 in \citet{uehara2021representation}.

\begin{lemma}[Simulation Lemma]\label{thm:simulation_lemma}
  Given two MDPs $(\P, r)$ and $(\hP, r)$, for any policy $\pi \in \varPi$, we have
  \begin{equation*}
    \rho_{\hP}(\pi) - \rhoP(\pi)
    = \frac{\gamma}{1-\gamma} \E[(s,a) \sim d_{\P}^{\pi}]{ \E[s' \sim \hP(\cdot | s,a)]{V_{\hP}^{\pi}(s')} - \E[s' \sim \P(\cdot | s,a)]{V_{\hP}^{\pi}(s')}}.
  \end{equation*}
\end{lemma}

\begin{proof}
  Note that, for any uniformly bounded function $f: \S \times \A \to \R$, we have
  \begin{align*}
    &\E[s \sim \mu_0, a \sim \pi(\cdot | s)]{f(s,a)} \\
    ={}& \frac{1}{1-\gamma} \E[\pi,\P]{ \sum_{t=0}^{\infty} \prn*{ \gamma^t f(s_t, a_t) - \gamma^{t+1} f(s_{t+1},a_{t+1}) } \;\middle\vert\; s_0 \sim \mu_0, a_0 \sim \pi(\cdot | s_0)} \\
    ={}& \frac{1}{1-\gamma} \sum_{s,a} f(s,a) \cdot \E[\pi,\P]{ \sum_{t=0}^{\infty} \prn*{ \gamma^t \ind{s_t = s, a_t = a} - \gamma^{t+1} \ind{s_{t+1} = s, a_{t+1} = a} } \;\middle\vert\; s_0 \sim \mu_0, a_0 \sim \pi(\cdot | s_0)} \\
    ={}& \frac{1}{1-\gamma} \sum_{s,a} f(s,a) \cdot \prn*{d_{\P}^{\pi}(s,a) - \gamma \sum_{\tilde{s}, \tilde{a}} d_{\P}^{\pi}(\tilde{s},\tilde{a}) \Ppi(s,a | \tilde{s},\tilde{a}) } \\
    ={}& \frac{1}{1-\gamma} \E[(s,a) \sim d_{\P}^{\pi}]{f(s,a) - \gamma \E[(s',a') \sim \Ppi(\cdot,\cdot | s,a)]{f(s',a')}}.
  \end{align*}
  Therefore, since $\rhoP(\pi) = \frac{1}{1-\gamma} \E[(s,a) \sim d_{\P}^{\pi}]{r(s,a)}$ and $\rho_{\hP}(\pi) = \Es[s \sim \mu_0, a \sim \pi(\cdot | s)]{Q_{\hP}^{\pi}(s,a)}$, we have
  \begin{align*}
    \rho_{\hP}(\pi) - \rhoP(\pi)
    &= \Es[s \sim \mu_0, a \sim \pi(\cdot | s)]{Q_{\hP}^{\pi}(s,a)} - \frac{1}{1-\gamma} \E[(s,a) \sim d_{\P}^{\pi}]{r(s,a)} \\
    &= \frac{1}{1-\gamma} \E[(s,a) \sim d_{\P}^{\pi}]{Q_{\hP}^{\pi}(s,a) - \gamma \E[(s',a') \sim \Ppi(\cdot,\cdot | s,a)]{Q_{\hP}^{\pi}(s',a')} - r(s,a)} \\
    &= \frac{\gamma}{1-\gamma} \E[(s,a) \sim d_{\P}^{\pi}]{ \E[(s',a') \sim \hPpi(\cdot,\cdot | s,a)]{Q_{\hP}^{\pi}(s',a')} - \E[(s',a') \sim \Ppi(\cdot,\cdot | s,a)]{Q_{\hP}^{\pi}(s',a')} },
  \end{align*}
  where in the last equality we plug in the Bellman equation $Q_{\hP}^{\pi}(s,a) = r(s,a) + \gamma \Es[(s',a') \sim \hPpi(\cdot,\cdot | s,a)]{Q_{\hP}^{\pi}(s',a')}$. Finally, we leverage the relationship between $Q$- and $V$-functions to complete the proof.
\end{proof}

\end{document}